\documentclass{article}
\usepackage[preprint]{neurips_2026}

\usepackage[utf8]{inputenc}
\usepackage[T1]{fontenc}
\usepackage{graphicx}
\usepackage{capt-of}
\usepackage{tikz}
\usetikzlibrary{positioning}

\usepackage{amsmath,amssymb,amsthm,mathtools,stmaryrd}
\usepackage{physics}
\usepackage{bbm}
\usepackage{dsfont}
\usepackage{booktabs}
\usepackage{nicefrac}
\usepackage{microtype}
\usepackage{url}
\usepackage{wrapfig}

\usetikzlibrary{arrows.meta,positioning,calc}

\usepackage[most]{tcolorbox}
\usepackage{xcolor}
\usepackage{float}
\usepackage{longtable}
\usepackage{listings}
\usepackage{multicol}

\definecolor{pipeblue}{HTML}{3B5BDB}
\definecolor{pipelight}{HTML}{EEF3FF}
\definecolor{pipegray}{HTML}{F3F4F6}
\definecolor{pipegreen}{HTML}{2E9D62}
\definecolor{pipedark}{HTML}{1F2933}
\definecolor{pipered}{HTML}{C2255C}
\definecolor{errorred}{HTML}{B42318}
\definecolor{acc1}{HTML}{1971C2}
\definecolor{acc2}{HTML}{3B5BDB}
\definecolor{acc3}{HTML}{2E9D62}
\definecolor{acc4}{HTML}{E8890C}
\definecolor{acc5}{HTML}{7048E8}

\lstdefinestyle{astgrammar}{
 basicstyle=\ttfamily\fontsize{7.5pt}{8.5pt}\selectfont,
 columns=fullflexible,
 keepspaces=true,
 showstringspaces=false,
 breaklines=true,
 breakatwhitespace=false,
 morekeywords={Program,Declaration,Set,Parameter,Variable,Piecewise,
 Constraint,SoftConstraint,PreferenceRelation,Objective,Formula,Expr,
 DataFormula,DataExpr,Domain,ParameterDomain,Basis,Aggregation,Relation,
 CompareOp,Op,NormOrder,ValueType,Value,Constant,ValueTable,SetValue,DataName,
 SetName,ParameterName,IndexName,Name},
 keywordstyle=\color{pipeblue!85!black}\bfseries,
 emph={minimize,maximize,feasibility,lexicographic,weighted_sum},
 emphstyle=\color{pipegreen!70!black}\bfseries,
 aboveskip=0pt,
 belowskip=0pt
}

\lstdefinestyle{appendixprompt}{
 basicstyle=\ttfamily\scriptsize,
 columns=fullflexible,
 keepspaces=true,
 showstringspaces=false,
 breaklines=true,
 breakatwhitespace=false,
 aboveskip=0pt,
 belowskip=0pt,
 literate={—}{{---}}1
}

\newtcblisting{grammarblock}[1]{
 enhanced,
 breakable,
 listing only,
 listing options={style=astgrammar},
 colback=pipegray!45,
 colframe=pipedark!35,
 colbacktitle=pipelight,
 coltitle=pipedark,
 fonttitle=\sffamily\bfseries\small,
 title={#1},
 boxrule=0.5pt,
 arc=1.5pt,
 left=6pt,
 right=6pt,
 top=3pt,
 bottom=3pt,
 before skip=3pt,
 after skip=4pt
}

\newtcblisting{promptblock}[1]{
 enhanced,
 breakable,
 listing only,
 listing options={style=appendixprompt},
 colback=pipegray!35,
 colframe=pipedark!30,
 colbacktitle=pipelight,
 coltitle=pipedark,
 fonttitle=\sffamily\bfseries\small,
 title={#1},
 boxrule=0.45pt,
 arc=1.5pt,
 left=6pt,
 right=6pt,
 top=5pt,
 bottom=5pt,
 before skip=6pt,
 after skip=8pt
}

\definecolor{annotblue}{HTML}{7A74E8}
\definecolor{annotpink}{HTML}{D96A8A}
\definecolor{treeblack}{HTML}{111111}
\colorlet{optboxcolor}{black}

\tikzset{
 treenode/.style={font=\bfseries\large\sffamily, text=treeblack, align=center},
 midsize/.style={font=\bfseries\normalsize\sffamily, text=treeblack, align=center},
 smallnode/.style={font=\normalsize\sffamily, text=treeblack, align=center},
 tinyann/.style={font=\bfseries\small\sffamily, text=annotblue, align=center},
 pinkann/.style={font=\bfseries\small\sffamily, text=annotpink, align=center},
 edge/.style={line width=1.1pt, draw=treeblack},
 annedge/.style={line width=1pt, draw=annotpink},
}

\newcommand{\astbranch}[2]{%
 \draw[astedge] ([xshift=2mm]#1.south west) |- ([xshift=-1mm]#2.west);%
}

\newcommand{\hrs}{\mathit{hours}}

\usepackage{enumitem}
\usepackage[hidelinks]{hyperref}
\hypersetup{colorlinks=true, linkcolor=blue, urlcolor=blue, citecolor=blue}

\theoremstyle{plain}
\newtheorem{theorem}{Theorem}[section]
\newtheorem{lemma}{Lemma}[section]
\newtheorem{corollary}{Corollary}[section]

\theoremstyle{definition}

\usepackage{paralist}
\usepackage{circledsteps}

\usetikzlibrary{calc,positioning,arrows.meta}
\definecolor{myblue}{RGB}{0,120,200}
\definecolor{newfill}{RGB}{252,233,220}
\definecolor{dashgray}{RGB}{190,190,190}

\title{ASTOR: Benchmarking and Improving Autoformulation for Soft-Constraint Optimization with a Formal Grammar}

\author{%
  Amer Goel, Yeganeh Alimohammadi, Andrew Daw \\
  Data Sciences and Operations \\
  Marshall School of Business \\
  University of Southern California
}
\begin{document}

\maketitle

\begin{abstract}
% {\color{red} I would have start with something like this "Many decision problems combine two qualitatively different kinds of requirements: hard constraints that must be satisfied, and soft constraints that express preferences among all the feasible solutions." then continuing with this is missing from nl to opt lit, and then saying we introduce ASTOR in the third sentence.
% }

Real-world decision problems combine two qualitatively different kinds of constraints: requirements that must be satisfied, called hard constraints, and preferences that can be violated, called soft constraints. Because these problems are often posed to decision-makers in natural language, large language models offer a natural tool for assisting with their solution. This has given rise to autoformulation, which focuses on using LLMs to translate natural-language descriptions of decision problems into executable mathematical optimization models. However, this research often neglects soft constraints like those proposed by consumers, managers, and planners, despite the fact that many practical problems require modeling them explicitly. We introduce \textbf{ASTOR}, an inference-time framework that maps a natural-language decision problem to a grammar-guided abstract syntax tree (AST) before deterministically compiling the resulting representation into executable optimization code. We also introduce \textbf{SoftOR}, a unified benchmark of natural language problems which focuses on soft constraint formulations. On SoftOR, ASTOR achieves the best aggregate accuracy among comparable single-request and multi-agent LLM methods while remaining competitive on standard autoformulation benchmarks that include only hard constraints. Code and data are available \href{https://anonymous.4open.science/r/astor-softor-iclr-E34C/README.md}{here}.

\end{abstract}

\section{Introduction}
\vspace{-6pt}
In practical settings, decision-makers often have to balance preferences from many stakeholders, some of which are not always satisfiable. These preferences do not constrain their decisions the same way that a strict requirement might. For example, a hospital manager must assign nurses to cover every shift, but may accommodate preferences to avoid night shifts or maintain consecutive days off. A travel agent must choose hotels and flights that obey the requested destination, but may not always be able to accommodate brand preferences. A shopping agent must obey the consumer’s budget, but may choose to suggest a longer delivery date than the one desired. An ideal optimization model must therefore distinguish what \textbf{must} hold from what is \textbf{preferred} to hold, and when preferences conflict, determine which take priority.
% {\color{red} (suggestion for replacing the last sentence):
% An ideal optimization model must therefore distinguish what \textbf{must} hold from what is \textbf{preferred} to hold, and when preferences conflict, determine which take priority.}

% These distinctions are difficult to recover from natural language. Managers and retailers may communicate priorities through natural language expressions, which are unstructured and therefore difficult to process. As a result, supporting these decisions requires translating preferences and priorities expressed in natural language into structured representations that can be analyzed computationally.

% Because of this natural language environment, LLMs are being used by consumers and firms alike to make decisions. 
Natural language is a familiar way to express these distinctions, and the growing use of large language models (LLMs) makes it an increasingly common interface for decision-making. Consumers use AI tools for shopping and planning \citep{ap_ai_shopping_2025, shopify_agentic_shopping, mckinsey_agentic_commerce}, while firms are incorporating AI for investments or as retail assistants \citep{bousquette_blackrock_ai_2026, mclymore_ai_retailers_2026}. These settings often present decision-makers with hard and soft constraints. As a result, LLM autoformulation provides a promising approach for this translation problem. Research in this field often uses LLM prompting or fine-tuning methods to successfully translate natural-language descriptions of real decision problems into solvable mathematical optimization formulations. However, their methods often prioritize only hard formulations \citep{cot, codegen, orthought} and current benchmarks only test those formulations \citep{nlp4opt,nlp4lp,complexor,industryor}. They do not directly test whether a model distinguishes requirements from preferences or preserves the ordering among preferences. This leaves an important class of practical decision problems without a strong autoformulation method or evaluation tools.

We address this gap with ASTOR and SoftOR. \textbf{ASTOR} is an inference-time method which prompts an LLM to express a decision problem as an abstract syntax tree generated by a formal grammar, then validates and compiles that representation deterministically into solver code. The constructed grammar is adequate for the class of hard and soft constraint optimization problems considered in this work; see Appendix~\ref{app:formal-adequacy}. \textbf{SoftOR} is our newly introduced \textbf{benchmark} for evaluating whether models can recover weighted-penalty and lexicographic soft-constraint formulations from natural language. It consists of tasks involving travel planning and scheduling nurses for shifts in a hospital. 
% Travel planning tasks consider preferences like hotel brand and transportation method, whereas scheduling nurses considers preferences for shift times and staffing requirements.

Across the 355 SoftOR instances, ASTOR achieves 88.5\% solution accuracy, compared with 84.8\% for the strongest alternative. It achieves the highest accuracy on three of the four SoftOR datasets, with the nurse scheduling dataset as the only exception. On 435 conventional hard-constraint problems, ASTOR achieves 81.8\%, within 0.5 percentage points of the best-performing method.
% Full results are available in Appendix~\ref{app:additional-results}.

\vspace{-6pt}
\section{Related Work}
\vspace{-6pt}
\label{sec:related-work}

\textbf{Autoformulation.}
Autoformulation methods differ in both their output representation and their use of inference-time computation. Under a single-request budget, Direct prompting returns a final solution, Codegen produces an executable solver program, and Chain-of-Thought prompting generates a natural-language derivation before producing the program \citep{codegen,cot,orthought}. 

Other systems improve autoformulation by adding training supervision or iterative inference. ORLM \citep{industryor} fine-tunes models on generated optimization instances, while LLMOPT~\citep{llmopt} trains models to produce structured formulations before code generation. Chain-of-Experts, OptiMUS, and ORThought instead distribute interpretation, formulation, implementation, and correction across multiple agents or execution-conditioned stages \citep{complexor,nlp4lp,orthought}. These methods demonstrate the value of additional supervision and refinement. Some methods use intermediate representations to organize formulation or code generation \citep{llmopt,nlp4lp}, but these representations are often not tailored to the task of optimization, nor do they admit any formal coverage guarantees. ASTOR uses the inference budget of the single-request methods and a bespoke formal representation designed for optimization formulations to produce more accurate models. 

\textbf{Benchmarking.}
Autoformulation benchmarks evaluate whether a system can translate a natural-language problem into a solvable mathematical program. Early benchmarks focused on recovering standard linear programs from relatively structured descriptions, as in NL4Opt \citep{nlp4opt}. Later work expanded the task to longer descriptions, mixed-integer programs, executable solver code, implicit constraints, and industrial problems \citep{nlp4lp,mamo,complexor,industryor,orthought}. Together, these benchmarks have moved evaluation toward more realistic and structurally diverse formulation tasks \citep{survey}.

Evaluation methods differ by benchmark, but the most common kind is called \textit{solution accuracy} \citep{llmopt, orthought, nlp4lp, complexor}. In this method, each natural language instance is labeled a priori with a gold-standard optimal \textit{value} that the authors consider correct. This optimal value can be a number if there is a single objective function, or a vector if there are multiple. To score an LLM autoformulation attempt, evaluators can take the optimization model produced by an LLM for a single natural language instance, perform the optimization using a solver like Gurobi, and compare the optimal value produced by the LLM's optimization formulation to the gold-standard value. The LLM's attempt is scored as correct if and only if the optimal value produced by its formulation matches the gold-standard value within the specified numerical tolerance. Other methods use an LLM and graph methods like embedding distance or edit distance to score the optimization formulation itself, instead of just the optimal value \cite{Opt-Verifier, SAC-Opt, VeriSimpl, gcn_paper, isomorphism}. This paper uses solution accuracy as the primary evaluation metric, as well as some measures of formulation success. 

Related benchmarks such as TravelPlanner, COMPASS, and SchedulingBenchmarks Nurse Rostering contain requirements and preferences that may conflict \citep{travelplanner,compass, nurse_bench}. However, these are not well prepared for autoformulation benchmarking, because none jointly provides complete natural language instances with gold-standard labeled optimal values. Additionally, existing autoformulation benchmarks focus primarily on feasibility constraints and scalar objectives. SoftOR is the first of its kind to directly target soft constraint formulations. 

\textbf{Formal Language in LLM Inference}
Formal languages have proven to be a useful intermediate step in other LLM translation tasks. Researchers have defined bespoke formal languages to successfully aid translation from natural language into formal logic or code, often representing the formulas produced in those languages through tree-graphs called  Abstract Syntax Trees (ASTs) \citep{ast_source,microsoft_syntax,microsoft_grammar}. For example, NL2LOGIC uses a Backus--Naur Form grammar and AST representation to generate logical expressions that are translated into Boolean-solver code, reporting up to 99\% syntactic faithfulness \citep{nl2logic}. AST-T5 and AST-based NL-to-SQL methods likewise use syntactic structure to support code generation \citep{ast_codegen, ast_sqlgen}, while AST-Probe finds recoverable code structure in language-model representations \citep{ast_llm_evidence}. 

Furthermore, the formal languages themselves often admit provable guarantees which are not accessible to other inference methods, and their computational representations as ASTs are useful for downstream formal evaluation such as graph embedding and isomorphism \citep{gcn_paper,isomorphism}. These results motivate ASTs as an interface between language models and executable formal systems, especially NL2LOGIC, which even uses ASTs for LLM translation and deterministic solving. ASTOR applies that interface to optimization and newly extends it to represent hard constraints, weighted penalties, and lexicographic priorities.

\vspace{-6pt}
\section{Methodology}
\vspace{-6pt}

\subsection{Formalizing Soft Preferences}

Consider a user who specifies a collection of requirements and preferences over decisions \(x\). We represent each requirement as a hard constraint \(h_j(x)\leq 0\), which every feasible solution must satisfy, and each preference \(i\) with a nonnegative violation function \(v_i(x)\), where \(v_i(x)=0\) indicates satisfaction and larger values indicate greater violation. The problem may also include a base objective \(f_0(x)\), such as monetary cost or total working time.

We consider two common formulations of these soft constraints:

$$
\begin{aligned}
\text{lexicographic:}\quad
&\operatorname*{lex\,min}_{x}\bigl(V_1(x),\ldots,V_K(x),f_0(x)\bigr),
&&V_k(x)=\sum_{i:\mathrm{Priority}(i)=k}v_i(x),\\[-1mm]
\text{weighted penalty:}\quad
&\min_x\;f_0(x)+\sum_i w_i v_i(x),
&&w_i\geq0,
\qquad\text{s.t. }h_j(x)\leq0.
\end{aligned}
\label{eq:soft-objectives1}
$$

 In the lexicographic formulation, each soft constraint is assigned to one of \(K\) priority tiers by \(\mathrm{Priority}:\mathbb{N}\to\mathbb{N}\), with lower indices indicating higher priority. These are indices that denote priority order in the mathematical formulation, operationally our method uses larger numbers to denote higher ranks, which is equivalent. The aggregate violation \(V_k(x)\) sums the violations in tier \(k\). The objective then minimizes these tier violations in priority order, so an improvement at a lower-priority tier cannot compensate for a loss at a higher-priority tier.

In the weighted-penalty formulation, each violation \(v_i\) instead receives a nonnegative weight \(w_i\). The weighted violations are added directly to the base objective, allowing preferences to trade off against one another and against the original objective \citep{golden2024unexpected,barmann2025lexicographic}. Both formulations are common in goal-programming applications such as scheduling and planning.

% The grammar $G$ declares finite sets, parameters, variables, and piecewise functions. Expressions combine arithmetic, elementary functions, norms, conditionals, and finite aggregations; formulas combine comparisons, Boolean connectives, and finite quantifiers. Distinct productions encode hard and soft constraints, preference relations, and weighted or lexicographic objectives. A grammar of this kind is called a Backus-Naur Form grammar, and is used often in formal logic literature to make an structured LLM output more reliable when an explicit syntax can be defined. 

\subsection{Formal Grammar}

The formal grammar \(G\) defines the intermediate language at the center of our method. $G$ provides compositional rules to assemble variables, domains, numbers, expressions, formulas, boolean connectives, aggregations, and explicit priority levels into formulas that represent optimization models. This grammar is given to our method as its output language, which guides it to create well-formed formulas in a language we know to be well-suited for optimization with soft constraints. Other methods, which output in natural language or solver-code, do not provide a bespoke formal representation designed for a specific optimization task. Our method will do this through the use of a formal grammar.

However, the grammar used by our method does achieve one such guarantee. In particular, we use a notion of \textit{adequacy}, which is a standard notion in formal logic for establishing that a formal language can represent its intended class of structures \citep{adequacy3,model_completeness}. In our case, \(G\) is \emph{adequate} for the class of optimization problems considered in this work when the correspondence holds in both directions: every well-formed formula generated by \(G\) denotes a problem in the class, and every problem in the class can be represented by a well-formed formula generated by \(G\). Thus, the grammar admits exactly the target problems without admitting any outside formulas \citep{adequacy3,model_completeness}. Appendix~\ref{app:formal-adequacy} gives the complete grammar, formal definition, and proof.

\vspace{-6pt}
\subsection{ASTOR Model Architecture}
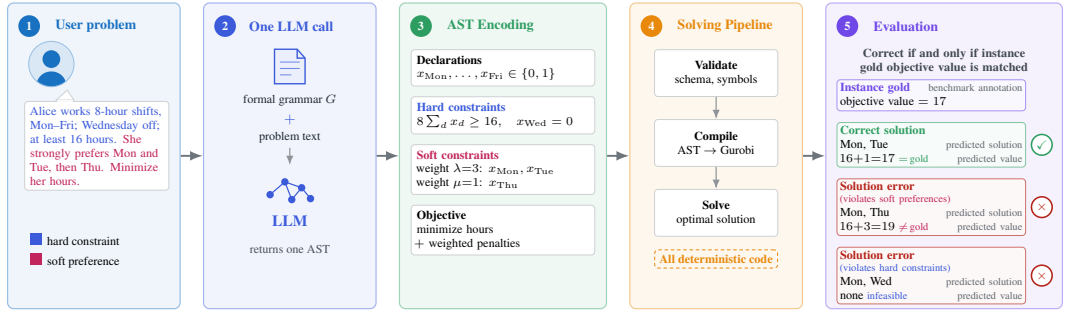
\begin{figure}[!htbp]
\centering
\resizebox{\textwidth}{!}{%
\begin{tikzpicture}[
 font=\sffamily,
 >=Latex,
 card/.style={draw=pipedark!22,line width=.5pt,rounded corners=2pt,fill=white,
 inner sep=3.2pt,align=left,font=\scriptsize},
 flow/.style={-{Latex[length=2.6mm,width=2mm]},line width=1.1pt,draw=pipedark!65},
 innerflow/.style={-{Latex[length=1.9mm,width=1.5mm]},line width=.7pt,draw=pipedark!50},
 glyph/.style={line width=.8pt,draw=pipedark},
]

\newcommand{\panelbox}[5]{%
 \node[fill=#4!7,draw=#4!45,line width=.7pt,rounded corners=3.5pt,inner sep=0pt,
 minimum width=#3,minimum height=62mm,anchor=north west] (#1) at (#2) {};
 \node[circle,fill=#4,text=white,font=\bfseries\scriptsize,minimum size=4.4mm,
 inner sep=0pt] (#1num) at ([shift={(4.5mm,-4.8mm)}]#1.north west) {#5};
}

\panelbox{p1}{0,0}{35mm}{acc1}{1}
\node[font=\bfseries\footnotesize,text=acc1,anchor=west] at ([xshift=1.8mm]p1num.east) {User problem};

\begin{scope}[shift={([shift={(9mm,-12.5mm)}]p1.north west)}]
 \fill[acc1!16] (0,0) circle (4.6mm);
 \begin{scope}
 \clip (0,0) circle (4.6mm);
 \fill[acc1] (0,1mm) circle (1.8mm);
 \fill[acc1] (-3.1mm,-4.8mm)
 .. controls (-3.1mm,-.4mm) and (3.1mm,-.4mm) .. (3.1mm,-4.8mm) -- cycle;
 \end{scope}
 \draw[acc1!45,line width=.5pt] (0,0) circle (4.6mm);
\end{scope}

\node[card,text width=27mm,anchor=north west] (req)
 at ([shift={(3.5mm,-20mm)}]p1.north west) {%
 {\color{pipeblue}Alice works 8-hour shifts, Mon--Fri; Wednesday off;
 at least 16 hours.}
 {\color{pipered}She strongly prefers Mon and Tue, then Thu.
 Minimize her hours.}};
\draw[fill=white,draw=pipedark!22,line width=.5pt]
 ([shift={(6.5mm,.2mm)}]req.north west) -- ++(2.2mm,2mm) -- ++(1.4mm,-2mm) -- cycle;
\draw[draw=white,line width=1pt]
 ([shift={(6.7mm,0mm)}]req.north west) -- ([shift={(9.9mm,0mm)}]req.north west);

\begin{scope}[font=\scriptsize]
\node[anchor=north west] (l1) at ([shift={(3.5mm,-45.5mm)}]p1.north west)
 {{\color{pipeblue}\rule{2.4mm}{2.4mm}}\, hard constraint};
\node[anchor=north west] at ([yshift=-4.2mm]l1.north west)
 {{\color{pipered}\rule{2.4mm}{2.4mm}}\, soft preference};
\end{scope}

\panelbox{p2}{[xshift=4.5mm]p1.north east}{35mm}{acc2}{2}
\node[font=\bfseries\footnotesize,text=acc2,anchor=west] at ([xshift=1.4mm]p2num.east) {One LLM call};

\begin{scope}[shift={([shift={(-3.1mm,-9.5mm)}]p2.north)}]
 \draw[glyph,draw=acc2,fill=acc2!10] (0,0) -- (4.6mm,0) -- (6.2mm,-1.8mm) -- (6.2mm,-7mm)
 -- (0,-7mm) -- cycle;
 \draw[glyph,draw=acc2] (4.6mm,0) -- (4.6mm,-1.8mm) -- (6.2mm,-1.8mm);
 \draw[acc2!75,line width=.5pt] (1.1mm,-2.6mm) -- (5.1mm,-2.6mm);
 \draw[acc2!75,line width=.5pt] (1.1mm,-3.9mm) -- (5.1mm,-3.9mm);
 \draw[acc2!75,line width=.5pt] (1.1mm,-5.2mm) -- (3.4mm,-5.2mm);
\end{scope}
\node[font=\scriptsize,anchor=north,text=pipedark]
 at ([shift={(0,-17.5mm)}]p2.north) {formal grammar $G$};
\node[font=\footnotesize\bfseries,text=acc2,anchor=north]
 at ([shift={(0,-21.5mm)}]p2.north) {$+$};
\node[font=\scriptsize,anchor=north,text=pipedark]
 at ([shift={(0,-25mm)}]p2.north) {problem text};
\draw[innerflow] ([shift={(0,-29.5mm)}]p2.north) -- ([shift={(0,-33mm)}]p2.north);

\begin{scope}[shift={([shift={(0,-38mm)}]p2.north)},line width=.7pt,draw=acc2]
 \draw (-4mm,-1mm) -- (-1mm,1.6mm) -- (2.4mm,.4mm) -- (4.2mm,-1.8mm)
 (-1mm,1.6mm) -- (.4mm,-2mm) -- (2.4mm,.4mm);
 \foreach \p in {(-4mm,-1mm),(-1mm,1.6mm),(2.4mm,.4mm),(4.2mm,-1.8mm),(.4mm,-2mm)}{
 \fill[acc2] \p circle (.75mm);}
\end{scope}
\node[font=\bfseries\small,text=acc2,anchor=north]
 at ([shift={(0,-42mm)}]p2.north) {LLM};
\node[font=\scriptsize,text=pipedark!80,anchor=north]
 at ([shift={(0,-48mm)}]p2.north) {returns one AST};

\panelbox{p3}{[xshift=4.5mm]p2.north east}{42mm}{acc3}{3}
\node[font=\bfseries\footnotesize,text=acc3,anchor=west] at ([xshift=1.8mm]p3num.east) {AST Encoding};

\node[card,text width=35mm,anchor=north] (astdecl)
 at ([yshift=-9.5mm]p3.north)
 {\textbf{Declarations}\\
 $x_{\mathrm{Mon}},\ldots,x_{\mathrm{Fri}}\in\{0,1\}$};
\node[card,text width=35mm,anchor=north] (asthard)
 at ([yshift=-2mm]astdecl.south)
 {\textbf{\color{pipeblue}Hard constraints}\\
 $8\sum_d x_d \geq 16,\quad x_{\mathrm{Wed}}=0$};
\node[card,text width=35mm,anchor=north] (astsoft)
 at ([yshift=-2mm]asthard.south)
 {\textbf{\color{pipered}Soft constraints}\\
 weight $\lambda{=}3$: $x_{\mathrm{Mon}},x_{\mathrm{Tue}}$\\
 weight $\mu{=}1$: $x_{\mathrm{Thu}}$};
\node[card,text width=35mm,anchor=north] (astobj)
 at ([yshift=-2mm]astsoft.south)
 {\textbf{Objective}\\
 minimize hours\\ $+$ weighted penalties};

\panelbox{p4}{[xshift=4.5mm]p3.north east}{35mm}{acc4}{4}
\node[font=\bfseries\footnotesize,text=acc4,anchor=west] at ([xshift=1.8mm]p4num.east) {Solving Pipeline};

\node[card,minimum width=25mm,minimum height=9.5mm,anchor=north] (val) at ([yshift=-9.5mm]p4.north) {};
\node[font=\scriptsize,align=left,anchor=west] at (val.center) [anchor=center,align=center]
 {\textbf{Validate}\\ schema, symbols};
\node[card,minimum width=25mm,minimum height=9.5mm,anchor=north] (comp) at ([yshift=-4.5mm]val.south) {};
\node[font=\scriptsize,align=left,anchor=west] at (comp.center) [anchor=center,align=center]
 {\textbf{Compile}\\ AST $\to$ Gurobi};
\node[card,minimum width=25mm,minimum height=9.5mm,anchor=north] (solve) at ([yshift=-4.5mm]comp.south) {};
\node[font=\scriptsize,align=left,anchor=west] at (solve.center) [anchor=center,align=center]
 {\textbf{Solve}\\ optimal solution};
\draw[innerflow] (val.south) -- (comp.north);
\draw[innerflow] (comp.south) -- (solve.north);
\node[draw=acc4,dash pattern=on 2pt off 1.6pt,rounded corners=2pt,inner sep=2.8pt,
 font=\scriptsize\bfseries,text=acc4,anchor=north]
 at ([yshift=-2.6mm]solve.south) {All deterministic code};

\panelbox{p5}{[xshift=4.5mm]p4.north east}{48mm}{acc5}{5}
\node[font=\bfseries\footnotesize,text=acc5,anchor=west] at ([xshift=1.8mm]p5num.east) {Evaluation};
\node[font=\scriptsize\bfseries,text=pipedark!85,anchor=north,
 text width=44mm,align=center]
 at ([yshift=-8.2mm]p5.north)
 {Correct if and only if instance\\ gold objective value is matched};

\node[card,draw=acc5!60,fill=acc5!4,text width=37mm,inner sep=1.8pt,anchor=north] (gold)
 at ([shift={(-2.6mm,-15.6mm)}]p5.north)
 {\textbf{\color{acc5}Instance gold}\hfill{\tiny\color{pipedark!70}benchmark annotation}\\
 objective value $=17$};

\node[card,draw=pipegreen!60,fill=pipegreen!4,text width=37mm,inner sep=1.8pt,anchor=north] (correct)
 at ([yshift=-2.2mm]gold.south)
 {\textbf{\color{pipegreen}Correct solution}\\
 Mon, Tue\hfill{\tiny\color{pipedark!70}predicted solution}\\
 $16{+}1{=}17$\;{\tiny\color{pipegreen}$=$\,gold}\hfill{\tiny\color{pipedark!70}predicted value}};
\node[circle,draw=pipegreen,line width=.9pt,minimum size=4.2mm,inner sep=0pt,
 fill=white] at ([xshift=3.2mm]correct.east)
 {\textcolor{pipegreen}{\scriptsize\checkmark}};

\node[card,draw=errorred!65,fill=errorred!3,text width=37mm,inner sep=1.8pt,anchor=north] (softwrong)
 at ([yshift=-2.2mm]correct.south)
 {\textbf{\color{errorred}Solution error}\\[-1pt]
 {\tiny\color{pipered}(violates soft preferences)}\\
 Mon, Thu\hfill{\tiny\color{pipedark!70}predicted solution}\\
 $16{+}3{=}19$\;{\tiny\color{pipered}$\neq$\,gold}\hfill{\tiny\color{pipedark!70}predicted value}};
\node[circle,draw=errorred,line width=.9pt,minimum size=4.2mm,inner sep=0pt,
 fill=white] at ([xshift=3.2mm]softwrong.east)
 {\textcolor{errorred}{\scriptsize$\times$}};

\node[card,draw=errorred!65,fill=errorred!3,text width=37mm,inner sep=1.8pt,anchor=north] (hardwrong)
 at ([yshift=-2.2mm]softwrong.south)
 {\textbf{\color{errorred}Solution error}\\[-1pt]
 {\tiny\color{pipeblue}(violates hard constraints)}\\
 Mon, Wed\hfill{\tiny\color{pipedark!70}predicted solution}\\
 none\;{\tiny\color{pipeblue}infeasible}\hfill{\tiny\color{pipedark!70}predicted value}};
\node[circle,draw=errorred,line width=.9pt,minimum size=4.2mm,inner sep=0pt,
 fill=white] at ([xshift=3.2mm]hardwrong.east)
 {\textcolor{errorred}{\scriptsize$\times$}};

\coordinate (mid) at ($(p1.north east)!0.5!(p1.south east)$);
\draw[flow] (p1.east |- mid) -- (p2.west |- mid);
\draw[flow] (p2.east |- mid) -- (p3.west |- mid);
\draw[flow] (p3.east |- mid) -- (p4.west |- mid);
\draw[flow] (p4.east |- mid) -- (p5.west |- mid);

\end{tikzpicture}%
}
\caption{ASTOR: One LLM call forms an AST which is compiled and solved deterministically.}
\label{fig:compact-pipeline}
\end{figure}

% Advisor experiment comments addressed in the tables and appendix:
% in the main paper, aren't we supposed to report the best in lit too?
% using the same 10,000 paired bootstrap samples, report 95\% confidence
% intervals for the difference.
% report paired ASTOR--CoE outcomes in RQ1.

\vspace{-6pt}
Our method prompts an LLM with a natural-language problem and the grammar $G$ to produce a well-formed expression in the language generated by $G$, which in this case is an optimization problem that may have hard or soft constraints. Every such expression admits a corresponding Abstract Syntax Tree (AST), a tree-structured representation that makes the compositional relationships between optimization components explicit. This gives our method a constrained intermediate representation between natural language and executable optimization code. AST-based intermediate representations have been used in natural-to-formal language translation to reduce generation errors and enforce compositional consistency \citep{nl2logic, ast_codegen,ast_sqlgen,ast_llm_evidence}. We adapt this approach to mathematical optimization, and develop a typed grammar, which is more expressive than some first-order-logic grammars developed in previous AST papers. For this reason, we name our method \textbf{ASTOR}: \textit{Abstract Syntax Trees for Operations Research.}

The specific implementation is in Figure~\ref{fig:compact-pipeline}. Concretely, the model is given a prompt including a brief description of the task, the formal grammar, and a JSON schema which serializes the AST and is the expected output form. The AST-JSON contains variable declarations, hard constraints, soft constraints, and priority annotations. A deterministic post-processing pipeline applies only local JSON repairs, validates and parses the resulting tree, and compiles it into executable Gurobi code. Thus, the LLM is responsible for formulating the optimization problem, while all subsequent translation from the formal representation to solver code is deterministic. The resulting solver output is evaluated against the gold optimum. Worked examples and execution traces are provided in Appendix~\ref{app:grammar}, with complete prompts and execution rules in Appendix~\ref{app:prompts}.

\vspace{-6pt}
\section{Building a New Benchmark: SoftOR}
\label{sec:softor}
\vspace{-6pt}
\vspace{-4pt}
Existing LLM autoformulation benchmarks primarily evaluate feasibility and objective recovery for problems expressed through hard constraints. Existing datasets contain preferences, but do not jointly provide natural-language instances, verified gold optima, and annotations for both weighted-penalty and lexicographic objectives. We therefore introduce SoftOR (Soft Operations Research), a unified natural-language soft-constraint autoformulation benchmark covering both objective structures. To construct and annotate lexicographic instances, we draw on the theory of \textbf{modality} from linguistics, which studies expressions such as \emph{must}, \emph{should}, and \emph{may} that describe necessity, obligation, and desire \citep{modal_bench, modal_explainer}. These provide an immediate natural-language interface for soft constraints, and linguistic theory already motivates an ordering on modal words by strength that can be used to define lexicographic levels.

SoftOR comprises Nurse Rostering, COMPASS, TravelPlanner, and LexiBench. For the first three components, we use published datasets as basic source material, but construct the natural-language instances, soft-constraint semantics, priority annotations, and evaluation rules required for autoformulation. LexiBench is entirely original. Table~\ref{tab:softor-composition} summarizes the resulting components and identifies the source data used in their construction.

\begin{table}[!htbp]
\vspace{-8pt}
\centering
\scriptsize
\caption{SoftOR components with new natural language problems and evaluations.} 
\label{tab:softor-composition}
\resizebox{\linewidth}{!}{%
\begin{tabular}{lrlcll}
\toprule \small
Component & $n$ & Objective & Priority Signal & Gold construction & Source Material \\
\midrule
Nurse Rostering & 15 & Weighted penalty & Explicit & Published optimum & \citet{nurse_bench,curtois_qu2014} \\
COMPASS & 85 & Lexicographic & Explicit & Enumeration & \citet{compass} \\
TravelPlanner & 180 & Lexicographic & Modal & Enumeration & \citet{travelplanner} \\
LexiBench & 75 & Lexicographic & Mixed & Enumeration/Solver Code & This work \\
\bottomrule
\end{tabular}%
}
\vspace{-8pt}
\end{table}

Full benchmark construction details, including natural language cues, are included in Appendix \ref{app:benchmark-construction}.

\subsection{Constructing New SoftOR Autoformulation Instances from Source data}

The source data from Nurse Rostering, COMPASS, and TravelPlanner were designed for scheduling or planning evaluation rather than natural-language optimization autoformulation and were therefore not directly usable as SoftOR instances without our adaptations as described below.

For Nurse Rostering, we convert published tabular requirements into new natural-language instances with explicit priority markers. Each published request weight becomes a fixed phrase. In particular, the original priority markers, annotated as 1, 2, 3, and 100, become ``slightly prefers," ``prefers," ``strongly prefers" and ``far more serious." The original numbers are not kept in the final natural language instance, since we intend to test whether these ordinal priorities can be recovered from language alone. Optima are provided by the original benchmark authors~\citep{nurse_bench}. Numerical weights are not given to any method. Each method instantiates its own weights, if it chooses to use them, and the specific numerical weights given by the benchmark are used only post-modeling for scoring and evaluation. We chose not to show weights to any route during decision-making so we can evaluate each route's ability to infer preference strength from natural language. As an ablation, we run a case where routes can see weights before decisions in Appendix \ref{app:benchmark-construction}.

The original COMPASS source material provided a natural-language query, structured hard constraints and preferences, a finite set of candidate plans, and one annotated gold plan ID \citep{compass}. It did not present these elements as a single autoformulation instance or annotate the underlying optimization objective. To adapt these instances for SoftOR, we combine the query and structured fields into a complete natural-language problem using fixed templates that mark requirements as "required" and wishes as equally important soft constraints that are "preferred". We also reconstruct the gold objective. Because multiple plans may satisfy the same maximum number of preferences, we enumerate every provided plan and use that maximum as gold rather than requiring the single plan ID supplied originally by COMPASS.

The original TravelPlanner source already provided natural-language travel-planning problems and structured candidate data, but it was designed to evaluate plan generation and does not supply annotated optimization formulations or gold objective vectors \citep{travelplanner}. For our adaptation, we convert each instance into a lexicographic autoformulation task by assigning its linguistic cues to three priority tiers. These cues, called \textbf{modals}, express obligation, intention, or desire, and the tiers follow the deontic--teleological--bouletic distinction in modal linguistics \citep{teleogical,corrob1,corrob2,modal_bench}: deontic language states what must hold, teleological language what the speaker plans to achieve, and bouletic language what the speaker wants. We then enumerate every valid combination of transport, one hotel per city, and restaurants; compare plans lexicographically by tier violations; and break ties by total cost. The resulting optimal vector is the gold annotation, and a formulation is correct only if its solution matches that vector at every tier and in cost. The complete inventory of linguistic cues and benchmark construction details is included in Appendix \ref{app:benchmark-construction}.

\subsection{Original Construction of LexiBench}

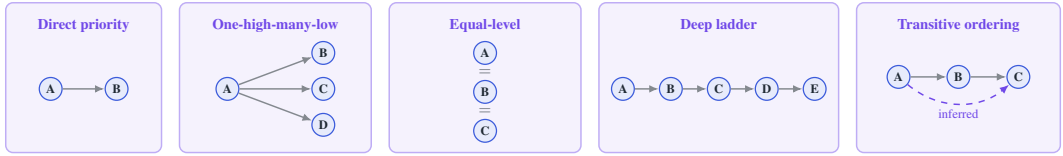
\begin{figure}[!htbp]
\centering
\resizebox{\textwidth}{!}{%
\begin{tikzpicture}[
 font=\sffamily,
 >=Latex,
 card/.style={draw=pipedark!22,line width=.5pt,rounded corners=2pt,fill=white,
 inner sep=2.4pt,align=left,font=\scriptsize},
 innerflow/.style={-{Latex[length=1.6mm,width=1.3mm]},line width=.7pt,draw=pipedark!55},
 pref/.style={circle,draw=acc2,fill=acc2!12,line width=.6pt,minimum size=3.8mm,inner sep=0pt,
 font=\tiny\bfseries,text=pipedark},
 ptitle/.style={font=\bfseries\footnotesize,text=pipedark,anchor=north west},
 panel/.style 2 args={fill=#1!7,draw=#1!45,line width=.7pt,rounded corners=3.5pt,inner sep=0pt,
 minimum width=#2,anchor=north west},
 badge/.style={circle,fill=#1,text=white,font=\bfseries\scriptsize,minimum size=4.4mm,inner sep=0pt},
 ctitle/.style={font=\bfseries\footnotesize,text=#1,anchor=west},
 chip/.style={rounded corners=2.4pt,fill=#1!14,draw=#1!35,line width=.4pt,text=#1!75!black,
 font=\scriptsize\bfseries,inner xsep=2.6pt,inner ysep=1.7pt},
 quote/.style={font=\scriptsize\itshape,text=pipedark,text width=27mm,align=left,inner sep=0pt},
 chipsmall/.style={chip=#1,font=\tiny\bfseries,inner xsep=2pt,inner ysep=1.5pt},
]

% ---------- (d) LexiBench preference structures ----------

\newcommand{\structcard}[4]{%
 \node[panel={acc5}{#3},minimum height=25mm] (#1) at (#2,-61.5mm) {};
 \node[font=\bfseries\scriptsize,text=acc5,anchor=north] at ([yshift=-1.6mm]#1.north) {#4};
}
\structcard{s1}{0mm}{26mm}{Direct priority}
\structcard{s2}{29mm}{32mm}{One-high-many-low}
\structcard{s3}{64mm}{32mm}{Equal-level}
\structcard{s4}{99mm}{40mm}{Deep ladder}
\structcard{s5}{142mm}{34mm}{Transitive ordering}

\node[pref] (s1a) at ([shift={(-5.5mm,-14.5mm)}]s1.north) {A};
\node[pref] (s1b) at ([shift={( 5.5mm,-14.5mm)}]s1.north) {B};
\draw[innerflow] (s1a) -- (s1b);
\node[pref] (s2a) at ([shift={(-8mm,-14.5mm)}]s2.north) {A};
\node[pref] (s2b) at ([shift={( 8mm,-8.5mm)}]s2.north) {B};
\node[pref] (s2c) at ([shift={( 8mm,-14.5mm)}]s2.north) {C};
\node[pref] (s2d) at ([shift={( 8mm,-20.5mm)}]s2.north) {D};
\draw[innerflow] (s2a) -- (s2b); \draw[innerflow] (s2a) -- (s2c); \draw[innerflow] (s2a) -- (s2d);
\node[pref] (s3a) at ([shift={(0mm,-8.5mm)}]s3.north) {A};
\node[pref] (s3b) at ([shift={(0mm,-15mm)}]s3.north) {B};
\node[pref] (s3c) at ([shift={(0mm,-21.5mm)}]s3.north) {C};
\node[font=\scriptsize,text=pipedark!70] at ($(s3a)!0.5!(s3b)$) {$=$};
\node[font=\scriptsize,text=pipedark!70] at ($(s3b)!0.5!(s3c)$) {$=$};
\node[pref] (s4a) at ([shift={(-16mm,-14.5mm)}]s4.north) {A};
\node[pref] (s4b) at ([shift={( -8mm,-14.5mm)}]s4.north) {B};
\node[pref] (s4c) at ([shift={( 0mm,-14.5mm)}]s4.north) {C};
\node[pref] (s4d) at ([shift={( 8mm,-14.5mm)}]s4.north) {D};
\node[pref] (s4e) at ([shift={( 16mm,-14.5mm)}]s4.north) {E};
\draw[innerflow] (s4a) -- (s4b); \draw[innerflow] (s4b) -- (s4c);
\draw[innerflow] (s4c) -- (s4d); \draw[innerflow] (s4d) -- (s4e);
\node[pref] (s5a) at ([shift={(-10mm,-12.5mm)}]s5.north) {A};
\node[pref] (s5b) at ([shift={( 0mm,-12.5mm)}]s5.north) {B};
\node[pref] (s5c) at ([shift={( 10mm,-12.5mm)}]s5.north) {C};
\draw[innerflow] (s5a) -- (s5b); \draw[innerflow] (s5b) -- (s5c);
\draw[innerflow,dashed,draw=acc5] (s5a) to[bend right=40] node[below,font=\tiny,text=acc5,yshift=1pt]{inferred} (s5c);

\end{tikzpicture}%
}
\vspace{-12pt}
\caption{Priority annotation in LexiBench instances. Arrows indicate priority ordering: left is a higher priority than right.}
\label{fig:priority-annotation-lexibench}
\end{figure}

LexiBench is an original dataset of 75 problems that tests richer ordered preferences than the preceding sources, which contain either fixed weights or at most three tiers. Figure~\ref{fig:priority-annotation-lexibench} shows its five preference structures. Each structure is instantiated for three decision types. \textit{Atomic} problems select one element from a candidate set, as in COMPASS. \textit{Assembly} problems build a bundle from several candidate sets, as in TravelPlanner. \textit{Constructive} problems formulate constraints and optimize over an implicitly defined feasible region, as in conventional hard-constraint benchmarks. Hardness is signaled with the same modal cues as TravelPlanner, and when more than three levels are needed, explicit comparatives such as ``more important than,'' ``higher priority,'' and ``above all'' fix the remaining order. Gold is computed by enumeration for atomic and assembly problems and by solving the annotated lexicographic model in Gurobi for constructive problems. A formulation is correct if it attains the gold objective vector, regardless of which optimal solution it returns.

The \emph{priority signal} column of Table~\ref{tab:softor-composition} records how each source communicates priority, and Appendix~\ref{app:additional-results} disaggregates accuracy by it. \emph{Explicit} sources state priority in words that name a level directly, such as ``strongly prefer'' and ``slightly prefer'' in Nurse Rostering or ``equally important'' in COMPASS. \emph{Modal} sources encode priority only through the modal verbs above, so the tier must be inferred from the strength of the wording. \emph{Mixed} sources, namely LexiBench, combine modal cues with explicit comparatives and pairwise statements, so recovering the ordering requires both readings. Appendix~\ref{app:benchmark-construction} gives the complete construction and annotation procedures for SoftOR.

\vspace{-6pt}
\section{Experiments}

To evaluate ASTOR, we compare it with single-request prompting methods and previously published agentic autoformulation methods across hard and soft constraint benchmarks. Our experiments address the following research questions:
\vspace{-5pt}
{\setdefaultleftmargin{1.5em}{1.5em}{2.5em}{3.5em}{4.5em}{5.5em}
\begin{compactitem}
 \item RQ1: How accurately do single-request and multi-agent inference methods solve soft-constraint benchmarks? How does this performance differ by lexicographic and weighted penalty objectives?
 \item RQ2: How does accuracy vary by priority signal and problem difficulty?
 \item RQ3: What technical and solution errors do the methods produce?
 \item RQ4: How expensive is ASTOR compared to other methods?
 \item RQ5: How do the ASTOR components and backbone affect performance?
\end{compactitem}}

\vspace{-4pt}
\subsection{Experimental Setup}
\label{sec:setup}
\vspace{-4pt}
For our main experiment, we evaluate on SoftOR using temperature-zero Gemini 3.7 Flash as the LLM backbone with the provider-default thinking configuration on most benchmarks. The only exception is Nurse Rostering, where all routes receive an increased thinking budget and use Gemini 3.8 Flash, due to its more complex formulations. Appendix \ref{app:additional-results} includes additional experiments on standard autoformulation benchmarks distributed by~\citet{orthought}, but these include almost exclusively hard constraints. A non-reasoning model, a non-Gemini reasoning model, and a larger Gemini reasoning model are included in later ablation studies on SoftOR and standard benchmarks. 

\textbf{Solution Accuracy}, as defined in Section \ref{sec:related-work}, is the primary evaluation metric, and other metrics are reported in Table \ref{tab:soft-satisfaction} for more granularity on which soft constraints are most often satisfied. Since lexicographic problems have multiple objectives, gold-standard labels are vectors, and LLM formulations are considered correct if and only if the entire lexicographic objective vector produced matches the gold-standard one within the specified numerical tolerance. Weighted-penalty problems have a single objective, so the gold standard label is a number, and similarly, formulations are correct if and only if the number produced matches the gold-standard value within the specified numerical tolerance. 

We compare against three single-request routes using the same backbone model. Direct is given the natural language problem and predicts the final answer with no intermediate representation. Code-generation (Codegen) is a single-request LLM prompting method which is given the natural language problem and outputs executable Gurobi code to solve the problem \citep{codegen}. This is the \textbf{main control} of the experiment, and tests whether a formal grammar provides any advantage over another formal intermediate representation, like solver code. Chain-of-Thought (CoT) asks the LLM to think in steps to produce a mathematical optimization model and executable Gurobi code \citep{cot}. ORThought uses agents to formulate the optimization problem as executable code and at most three repair agents to debug the formulation or code \citep{orthought}. Chain-of-Experts (CoE) uses a panel of many LLM agents to collaborate on formulation, coding, and natural-language extraction \citep{complexor}. ORThought and CoE use published prompts. 

Direct, Codegen, and CoT are single-request LLM prompting methods, whereas CoE and ORThought are multi-request pipelines that use different LLM prompts for formulation, repair, and execution. We estimate sampling uncertainty with 10,000 paired bootstrap resamples over problem instances, using the same sampled instances for every method. Tables report bootstrap standard errors in parentheses, and figures show percentile 95\% confidence intervals. Appendix~\ref{app:prompts} gives the prompts. 

% \vspace{-4pt}
% \subsection{Datasets}
% \vspace{-4pt}
% Our primary evaluation is the 355 SoftOR instances of Section~\ref{sec:softor}: COMPASS (85), LexiBench (75), Nurse Rostering (15), and TravelPlanner (180). Appendix~\ref{app:additional-results} reports all six methods on 435 conventional hard-constraint instances from the cleaned ComplexOR, IndustryOR, LogiOR, and NLP4LP releases distributed with ORThought~\citep{orthought}, together with published results for other autoformulation systems.

\vspace{-4pt}
\subsection{RQ1: Accuracy}

Overall, ASTOR is the most accurate on SoftOR, with an aggregate accuracy across lexicographic and weighted penalty problems of 88.5\%. From most to least accurate, the baselines achieve 84.8\% for CoE, 83.7\% for Codegen, 81.4\% for CoT, 80.0\% for ORThought, and 52.7\% for Direct. We now report results stratified by objective type (lexicographic or weighted), including more granular data on soft constraint satisfaction. 

\subsubsection{Lexicographic Problems}
\vspace{-4pt} 

\begin{table}[!htbp]
\centering
\small
\begin{minipage}[t]{0.88\linewidth}
\vspace{0pt}
\centering
\begin{minipage}[t][30pt][t]{\linewidth}
\caption{Solution accuracy by dataset for six routes (\%), bootstrap standard errors are in parentheses. The best result is bolded.}
\label{tab:main-accuracy}
\end{minipage}
\centering
\fontsize{8.5}{10}\selectfont
\renewcommand{\arraystretch}{1.1}
\setlength{\tabcolsep}{1.5pt}
\begin{tabular*}{\linewidth}{@{\extracolsep{\fill}}lrrrr@{}}
\toprule
\rule[-4pt]{0pt}{20pt}Method & COMPASS & \shortstack{Travel\\Planner} & LexiBench & Overall \\
\midrule
Direct & 56.5 (5.4) & 75.0 (3.2) & 5.3 (2.6) & 55.0 (2.7) \\
Codegen & \textbf{98.8 (1.2)} & 76.1 (3.2) & 89.3 (3.5) & 84.7 (1.9) \\
CoT & 95.3 (2.3) & 77.2 (3.1) & 80.0 (4.6) & 82.4 (2.1) \\
ORThought & 96.5 (2.0) & 80.0 (3.0) & 66.7 (5.4) & 81.2 (2.1) \\
CoE & \textbf{98.8 (1.2)} & 79.4 (3.0) & 89.3 (3.5) & 86.5 (1.8) \\
ASTOR (Ours) & \textbf{98.8 (1.2)} & \textbf{85.0 (2.7)} & \textbf{92.0 (3.2)} & \textbf{90.0 (1.6)} \\
\bottomrule
\end{tabular*}
\end{minipage}
\end{table}

%ASTOR achieves the highest aggregate accuracy on SoftOR and leads on three of its four benchmarks. Full experiment details are in Appendix~\ref{app:prompts}.

In Table \ref{tab:main-accuracy}, we compare the six routes on solution accuracy over each component of SoftOR. Table~\ref{tab:rq1-hard-oneshot} in Appendix~\ref{app:additional-results} reports solution accuracy on the hard-constraint benchmarks.

\begin{table}[!htbp]
\centering
\small
\caption{Average soft constraint satisfaction (\%) over instances containing preferences in the relevant column. Combined pools tiers within each instance before averaging.}
\vspace{0.2cm}
\label{tab:soft-satisfaction}
\resizebox{\linewidth}{!}{%
\begin{tabular}{l c ccc cccc}
\toprule
 & COMPASS & \multicolumn{3}{c}{TravelPlanner} & \multicolumn{4}{c}{LexiBench} \\
\cmidrule(lr){2-2}\cmidrule(lr){3-5}\cmidrule(lr){6-9}
Method
& \shortstack{Preferences\\$n=85$}
& \shortstack{Tier 2\\Strong preferences\\$n=16$}
& \shortstack{Tier 3\\Weak preferences\\$n=87$}
& \shortstack{Combined\\$n=95$}
& \shortstack{Tier 1\\$n=75$}
& \shortstack{Tier 2\\$n=60$}
& \shortstack{Tiers 3--5\\$n=66$}
& \shortstack{Combined\\$n=75$} \\
\midrule
Direct       & 35.0          & 93.8          & 86.2          & 87.4          & 4.5           & 3.3           & 1.5           & 3.5 \\
Codegen      & \textbf{62.7} & 93.8          & 83.9          & 85.3          & 92.8          & 52.1          & 65.2          & 65.6 \\
CoT          & 60.4          & 81.2          & 85.1          & 85.3          & 90.1          & 51.8          & 60.6          & 63.8 \\
ORThought    & 60.7          & 93.8          & 83.9          & 85.3          & 73.1          & 38.6          & 50.0          & 52.0 \\
CoE          & \textbf{62.7} & \textbf{100.0}& 87.4          & 88.4          & \textbf{94.1} & 50.4          & \textbf{68.2} & \textbf{66.3} \\
ASTOR (Ours) & \textbf{62.7} & \textbf{100.0}& \textbf{92.0} & \textbf{92.6} & 92.0          & \textbf{53.2} & 63.6          & 63.5 \\
\bottomrule
\end{tabular}%
}
\vspace{-8pt}
\end{table}

In Table \ref{tab:soft-satisfaction}, “Combined” columns report the average across instances of the percentage of preferences satisfied within each instance, counting preferences from all tiers together.  Individual tier columns average only over instances containing preferences at that tier. All columns report proportion of preferences satisfied on each lexicographic level. LexiBench tiers are grouped for brevity and data sparsity. ASTOR performs best on COMPASS, TravelPlanner, and some LexiBench metrics. Hard constraint satisfaction on SoftOR is reported in Appendix~\ref{app:additional-results}.

\vspace{-4pt}
\subsubsection{Weighted Penalty Problems (Nurse Rostering Instances)}

\begin{table}[!htbp]
\vspace{-8pt}
\centering
\definecolor{passgreen}{HTML}{22834A}
\definecolor{midyellow}{HTML}{C49A00}
\definecolor{failred}{HTML}{B52D36}

\caption{Performance relative to gold-standard optimum on the 15 Nurse Rostering instances.}
\label{tab:nrp-penalty-above-gold}

\centering
\fontsize{7.5}{10}\selectfont
\setlength{\tabcolsep}{2pt}
\renewcommand{\arraystretch}{1.1}

\begin{tabular*}{\linewidth}{@{\extracolsep{\fill}}lrrrrrrrrrrrrrrr@{}}
\toprule
\rule[-3pt]{0pt}{10pt}Method & 1 & 2 & 3 & 4 & 5 & 6 & 7 & 8 & 9 & 10 & 11 & 12 & 13 & 14 & 15 \\
\midrule
Direct
& \textcolor{midyellow}{+407}
& \textcolor{midyellow}{+897}
& \textcolor{failred}{FAIL} & \textcolor{failred}{FAIL} & \textcolor{failred}{FAIL}& \textcolor{failred}{FAIL}&\textcolor{failred}{FAIL} &\textcolor{failred}{FAIL} &\textcolor{failred}{FAIL} & \textcolor{failred}{FAIL}& \textcolor{failred}{FAIL}&\textcolor{failred}{FAIL} &\textcolor{failred}{FAIL} &\textcolor{failred}{FAIL} & \textcolor{failred}{FAIL} \\

Codegen
& \textcolor{passgreen}{0}
& \textcolor{passgreen}{0}
& \textcolor{passgreen}{0}
& \textcolor{midyellow}{+1}
& \textcolor{midyellow}{+1}
& \textcolor{passgreen}{0}
& \textcolor{passgreen}{0}
& \textcolor{midyellow}{+1}
& \textcolor{midyellow}{+2}
& \textcolor{passgreen}{0}
& \textcolor{passgreen}{0}
& \textcolor{midyellow}{+1}
& \textcolor{passgreen}{0}
& \textcolor{midyellow}{+1}
& \textcolor{passgreen}{0} \\

CoT
& \textcolor{passgreen}{0}
& \textcolor{passgreen}{0}
& \textcolor{passgreen}{0}
& \textcolor{midyellow}{+1}
& \textcolor{midyellow}{+1}
& \textcolor{passgreen}{0}
& \textcolor{passgreen}{0}
& \textcolor{midyellow}{+6}
& \textcolor{midyellow}{+2}
& \textcolor{passgreen}{0}
& \textcolor{passgreen}{0}
& \textcolor{midyellow}{+1}
& \textcolor{passgreen}{0}
& \textcolor{midyellow}{+1}
& \textcolor{passgreen}{0} \\

ORThought
& \textcolor{passgreen}{0}
& \textcolor{passgreen}{0}
& \textcolor{passgreen}{0}
& \textcolor{midyellow}{+1}
& \textcolor{midyellow}{+3}
& \textcolor{failred}{FAIL}
& \textcolor{passgreen}{0}
& \textcolor{midyellow}{+1}
& \textcolor{midyellow}{+2}
& \textcolor{passgreen}{0}
& \textcolor{passgreen}{0}
& \textcolor{midyellow}{+1}
& \textcolor{passgreen}{0}
& \textcolor{failred}{FAIL}
& \textcolor{passgreen}{0} \\

CoE
& \textcolor{passgreen}{0}
& \textcolor{passgreen}{0}
& \textcolor{passgreen}{0}
& \textcolor{midyellow}{+1}
& \textcolor{midyellow}{+1}
& \textcolor{passgreen}{0}
& \textcolor{midyellow}{+1}
& \textcolor{midyellow}{+2}
& \textcolor{midyellow}{+2}
& \textcolor{passgreen}{0}
& \textcolor{passgreen}{0}
& \textcolor{midyellow}{+1}
& \textcolor{midyellow}{+2}
& \textcolor{midyellow}{+1}
& \textcolor{passgreen}{0} \\

ASTOR
& \textcolor{passgreen}{0}
& \textcolor{passgreen}{0}
& \textcolor{passgreen}{0}
& \textcolor{midyellow}{+1}
& \textcolor{midyellow}{+1}
& \textcolor{midyellow}{+1}
& \textcolor{passgreen}{0}
& \textcolor{midyellow}{+2}
& \textcolor{midyellow}{+2}
& \textcolor{passgreen}{0}
& \textcolor{passgreen}{0}
& \textcolor{midyellow}{+1}
& \textcolor{passgreen}{0}
& \textcolor{midyellow}{+4}
& \textcolor{passgreen}{0} \\
\bottomrule
\end{tabular*}
\end{table}

Table \ref{tab:nrp-penalty-above-gold} shows the performance of each method on each instance of Nurse Rostering. Numbers indicate incurred surplus penalty by violating staffing requirements or shift preferences. Zero penalty is an optimal solution. FAIL denotes instances where the route failed to obey all hard constraints or failed to produce a roster at all. Codegen and CoT are the most accurate on this dataset, with 9/15 correct instances (0 penalty). ASTOR and ORThought have 8/15, CoE has 7/15, and Direct has 0/15. ASTOR is as good as all other inference methods on staffing measures. It is neither the worst nor the best on weighted shift preferences. Details on Nurse Rostering are available in Appendix~\ref{app:benchmark-construction}.

\begin{figure}[b]
\vspace{-6pt}
\centering
\begin{minipage}[t]{0.32\linewidth}
\centering
\begin{minipage}[c][1.25in][c]{\linewidth}
\centering
\includegraphics[width=\linewidth,height=1.25in,keepaspectratio]{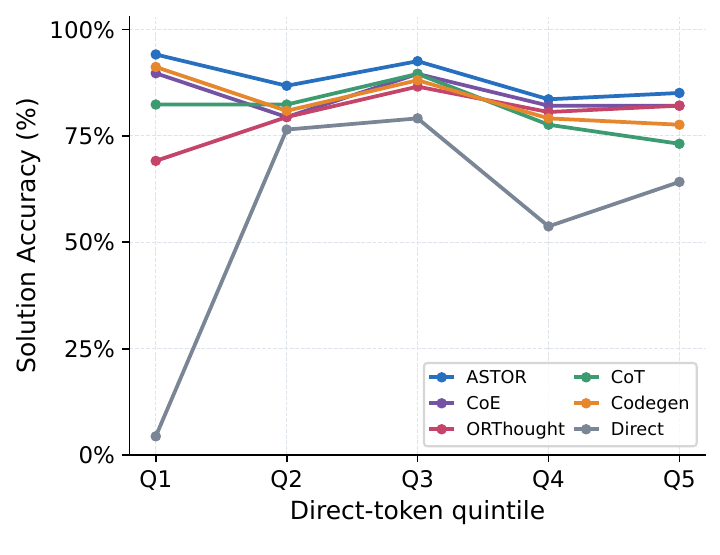}
\end{minipage}
\par\vspace{2pt}(a)
\end{minipage}\hfill
\begin{minipage}[t]{0.32\linewidth}
\centering
\begin{minipage}[c][1.25in][c]{\linewidth}
\centering
\includegraphics[width=\linewidth,height=1.25in,keepaspectratio]{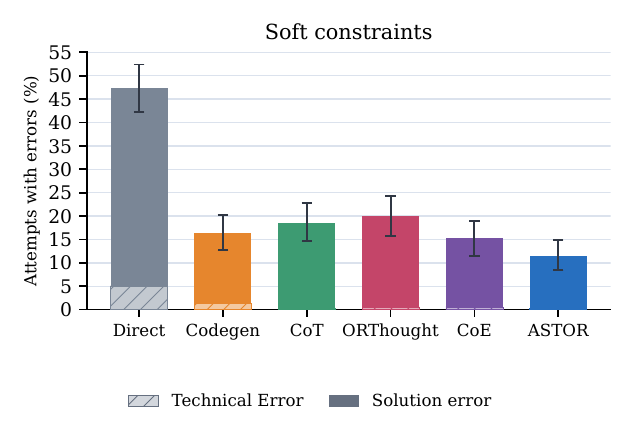}
\end{minipage}
\par\vspace{2pt}(b)
\end{minipage}\hfill
\begin{minipage}[t]{0.32\linewidth}
\centering
\begin{minipage}[c][1.25in][c]{\linewidth}
\centering
\includegraphics[width=\linewidth,height=1.25in,keepaspectratio]{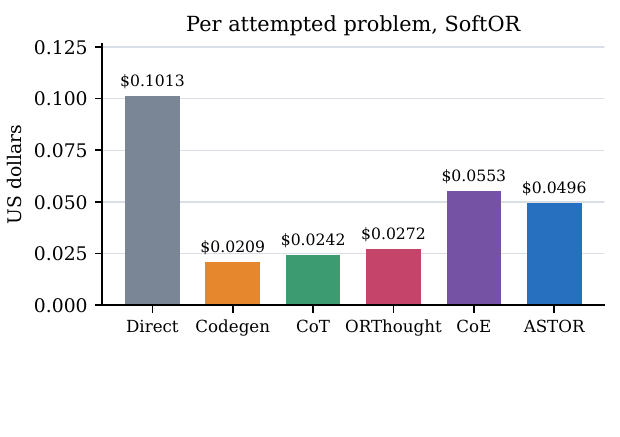}
\end{minipage}
\par\vspace{2pt}(c)
\end{minipage}
\vspace{-6pt}
\caption{(a) Accuracy by token quintile as a proxy for difficulty. (b) Error rates on SoftOR (\%) with 95\% CI. (c) Monetary cost in USD by route at time of evaluation.}
\makeatletter\begingroup\edef\@currentlabel{\thefigure a}\label{fig:rq2-soft-characteristics}\endgroup\makeatother
\makeatletter\begingroup\edef\@currentlabel{\thefigure b}\label{fig:error-analysis-soft}\endgroup\makeatother
\makeatletter\begingroup\edef\@currentlabel{\thefigure c}\label{fig:cost-soft-attempt}\endgroup\makeatother
\vspace{-6pt}
\end{figure}

\subsection{RQ2: Performance by Problem Characteristics}
\label{sec:rq2}
\vspace{-4pt}

\begin{wrapfigure}{r}{0.45\textwidth}
    \vspace{-6pt}
    \centering
    \includegraphics[width=\linewidth]{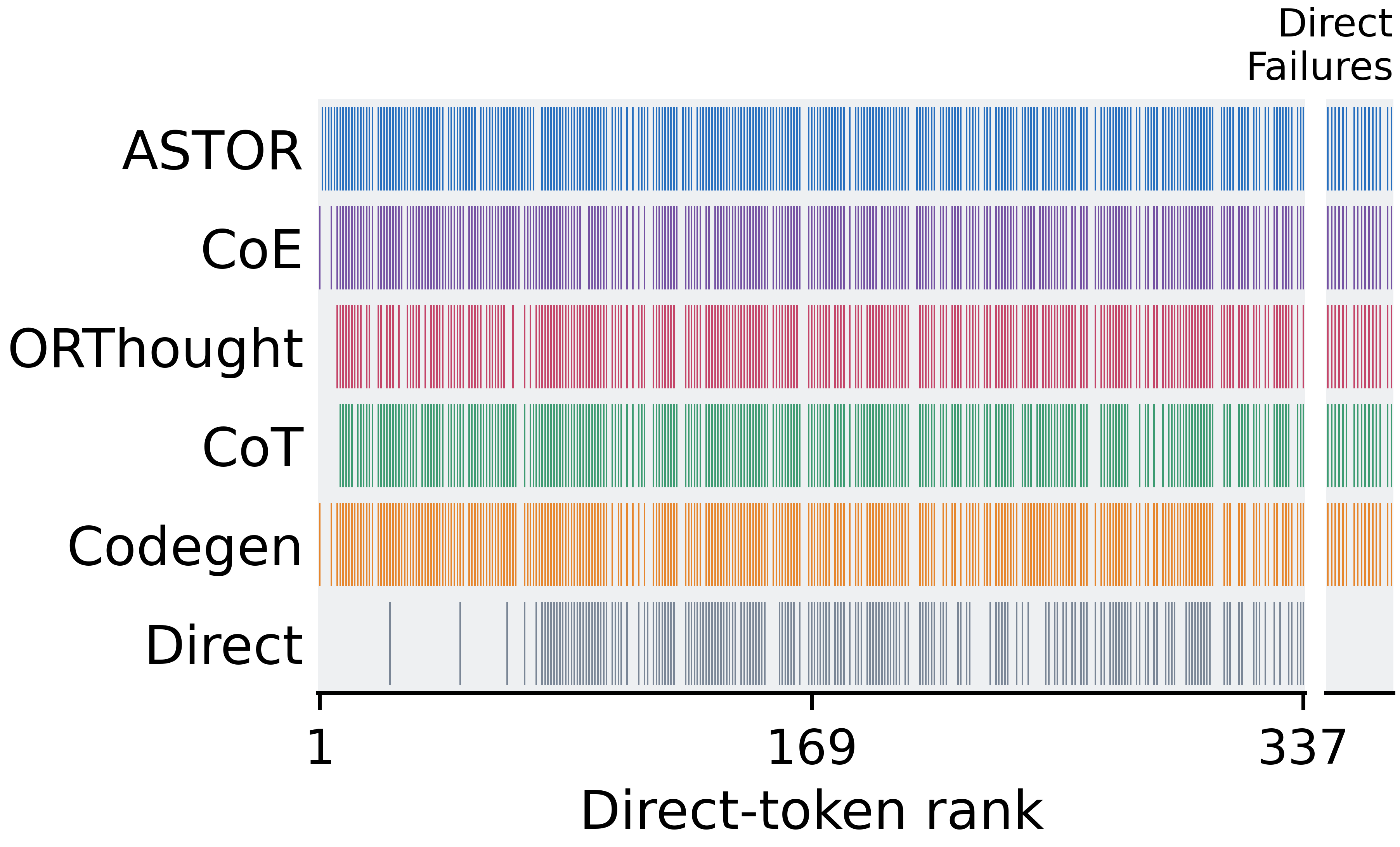}
    \caption{Correct instances ordered by Direct token usage as a proxy for difficulty.}
    \label{fig:token-tallies}
\end{wrapfigure}

We disaggregate results by difficulty level, with token usage as a proxy. Figure \ref{fig:token-tallies} has all 355 SoftOR instances as its horizontal axis, ordered from least to most token usage on the Direct route. Token usage counts how many tokens the Direct route used to produce an answer for that instance. Each tally mark indicates a correct answer for that instance. The section on the right, called "Direct Failures" is for those instances Direct did not produce a roster, and does not have token ranks on its horizontal axis. In this case, instances are sorted alphabetically by dataset name and then instance ID, and tally marks indicate correct solutions for each route. Figure \ref{fig:rq2-soft-characteristics} uses the same definition of token usage, and aggregates them into five roughly equal groups, called quintiles, with Q1 containing the lowest token usage and Q5 the highest. ASTOR performs best, even as the difficulty level of the problem increases. More disaggregated results are in Appendix~\ref{app:additional-results} (Figure~\ref{fig:rq2-hard-characteristics}).

\subsection{RQ3: Error Analysis}
\vspace{-4pt}
We analyze technical and solution errors as percentages of all evaluated attempts. Technical errors include malformed outputs, parsing, compilation, or execution failures. Solution errors indicate that the optimal value is different than the gold-standard. More details are in Appendix~\ref{app:error-taxonomy}.

ASTOR has the lowest error rate (i.e., percent of instances with either technical or solution errors) on SoftOR among the compared methods (Figure~\ref{fig:error-analysis-soft}). The whiskers show 95\% bootstrap confidence intervals for the total error rate, combining technical and solution errors, and the unplotted remainder up to 100\% represents correct attempts. On LexiBench, Codegen and CoT often satisfy the right preferences but fail to choose the cheapest solution among equally preferred options, accounting for 5 of Codegen’s 8 solution errors and 10 of CoT’s 12. ASTOR instead gives up preferences to reduce cost in five of its six errors because it combines preference penalties and cost in a weighted sum. In the remaining case, it gives the highest priorities to the wishes the request explicitly says to give up first. On the hard-constraint benchmarks, error rates are similar for ASTOR, Codegen, CoT, ORThought, and CoE, for which full details are included in Appendix~\ref{app:additional-results}.

\vspace{-4pt}
\subsection{RQ4: Cost Analysis}
\vspace{-4pt}

We report the incurred cost in USD of ASTOR and other methods at the time of evaluation. Figure \ref{fig:cost-soft-attempt} shows ASTOR is 10.3\% less expensive than CoE, the strongest alternative, while also being the most accurate. Other cost results are in Appendix \ref{app:additional-results}, where ASTOR shows a similar profile.

\vspace{-4pt}
\subsection{RQ5: Ablation Studies}
\label{sec:rq4}
\vspace{-4pt}
We examine the prompt components and LLM backbone of ASTOR. Following precedent for analysis on fixed benchmark subsets~\citep{nemo,OptMATH,seakr}, we construct a frozen, stratified 100-instance sample drawn from the eight benchmarks, including the four from SoftOR and four hard-constraint-only benchmarks in Appendix \ref{app:additional-results}. All main-text ablations report results on the 45 soft-constraint instances included in the sample, while Appendix~\ref{app:additional-results} reports the corresponding results on the complementary 55 hard-constraint instances. For every ablation except the LLM backbone, the same settings as the main experiment are applied to each dataset.

As seen in Appendix \ref{app:prompts}, the ASTOR prompt has several components, and we remove one for each ablation. The ASTOR prompt includes the formal grammar in two parts, one as a human-readable markdown file in the style of a BNF grammar, and the other as a JSON schema that specifies the expected form of the output. It also supplies a benchmark contract, with bespoke information needed to use the data of each benchmark. Lastly, ASTOR uses a deterministic Python code repair which detects and fixes syntax mistakes in the JSON output, for which more information is available in Appendix \ref{app:additional-results}.

\IfFileExists{generated/reporting/ablation-components-soft.pdf}{%
\begin{figure}[!htbp]
\centering
\includegraphics[width=0.65\linewidth]{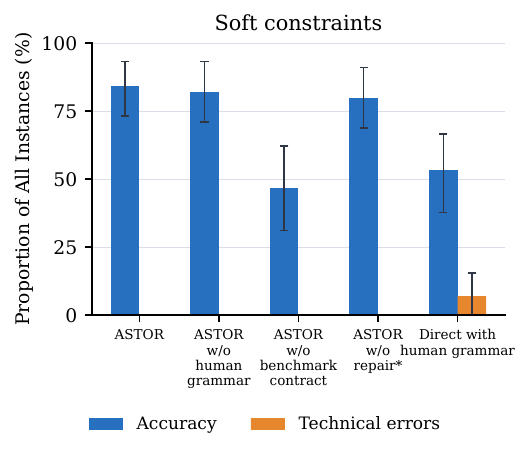}
\vspace{-8pt}
\caption{ASTOR prompt ablations on SoftOR. Whiskers show bootstrap 95\% confidence intervals.}
\label{fig:ablation-soft}
\vspace{-8pt}
\end{figure}
}{}

Removing the human-readable grammar and the deterministic repair reduced performance by only 2.2 and 4.4 percentage points, respectively, indicating that the serialized grammar may be enough and that syntax errors are uncommon. Removing the benchmark contract reduces performance by more. The last bar in Figure \ref{fig:ablation-soft} gives the human grammar to the Direct method, provided as a tool for internal reasoning only. This method performs exactly the same as the Direct method on the ablation set.

\begin{table}[H]
\vspace{-6pt}
\centering
\small
\caption{Model-backbone ablation on the 45 soft-constraint instances. Solution accuracy (\%); bootstrap standard errors are in parentheses.}
\label{tab:ablation-reasoning}
\resizebox{0.92\linewidth}{!}{%
\begin{tabular}{lccc}
\toprule
Method & \shortstack{Gemini 3.5 Flash Lite} & \shortstack{GPT-5} & \shortstack{Gemini 3.1 Pro} \\
\midrule
Direct & 6.7 (3.7) & 44.4 (7.4) & 33.3 (7.0) \\
Codegen & \textbf{44.4 (7.4)} & 68.9 (6.9) & \textbf{80.0 (6.0)} \\
CoT & 17.8 (5.7) & 51.1 (7.5) & 73.3 (6.6) \\
ORThought & 28.9 (6.8) & 68.9 (6.9) & \textbf{80.0 (6.0)} \\
CoE & 37.8 (7.2) & 68.9 (6.9) & 77.8 (6.2) \\
ASTOR (Ours) & 26.7 (6.6) & \textbf{73.3 (6.6)} & \textbf{80.0 (6.0)} \\
\bottomrule
\end{tabular}}
\vspace{-8pt}
\end{table}

\vspace{0.2cm}

 Codegen leads on Gemini 3.5 Flash Lite, whereas ASTOR leads on the GPT backbone. Codegen, ORThought, and ASTOR tie on Gemini 3.1 Pro, with CoE close behind. (Table~\ref{tab:ablation-reasoning}).

\vspace{-6pt}
\section{Discussion and Conclusion}
\vspace{-6pt}
SoftOR makes soft-constraint autoformulation measurable through weighted penalties, lexicographic objectives, and solver-verified gold annotations, and ASTOR shows that a formal intermediate language improves accuracy on exactly those problems, particularly ordered preferences, while remaining competitive on hard constraints. Future work should define and evaluate on a metric of faithfulness to the original optimization problem or solver performance and extend the benchmark with less explicit language to test how well the method transfers to those settings. Overall, the results support formal languages and ASTs as practical foundations for optimization autoformulation. 

\newpage

\subsubsection*{Ethics Statement}
This work does not involve human subjects, personal data, or sensitive attributes, so institutional review board approval was not required. SoftOR is built from publicly released benchmarks (COMPASS, TravelPlanner, and SchedulingBenchmarks) under their original licenses, together with LexiBench, an original dataset written and annotated by the authors. All priority annotations, gold constructions, and error labels were produced by the authors following the protocols in Appendices~\ref{app:error-taxonomy} and~\ref{app:benchmark-construction}, and no crowdsourcing or paid annotation was used. The natural-language instances describe scheduling, travel, and staffing scenarios with fictional entities and contain no real individuals. Experiments call commercial LLM APIs, and the available token usage and cost of every reported run are recorded in Appendix~\ref{app:additional-results} at the prices listed in Appendix~\ref{app:prompts} so that the environmental and financial footprint of the study is visible. The methods and benchmark are intended to make optimization formulation more faithful to stated preferences and carry no foreseeable dual-use risk beyond that of general-purpose optimization tooling. The authors have no conflicts of interest or sponsorship to declare.

\subsubsection*{Reproducibility Statement}
We release the complete ASTOR pipeline, the SoftOR benchmark with its gold annotations, all prompts, and the raw outputs and evaluation records of every reported run as supplementary material and at the repository linked in the abstract. The formal grammar, the definition of the problem class, and the full proof of adequacy are in Appendix~\ref{app:formal-adequacy}, and Appendix~\ref{app:grammar} works one instance through the pipeline with its literal compiler trace. Appendix~\ref{app:prompts} gives the exact prompts for the four single-request routes, the configuration and prompt provenance of ORThought and CoE, the benchmark contracts, the execution and evaluation rules, the frozen decoding and pricing settings, and the ablation definitions. The identifiers of the frozen 100-instance ablation subset are available through the repository. Appendix~\ref{app:benchmark-construction} documents every dataset adaptation, including the TravelPlanner priority rules, the LexiBench construction, the Nurse Rostering phrase mapping, the COMPASS templates, and the gold-label corrections applied to the ORThought releases. Section~\ref{sec:setup} states the backbone, decoding settings, and bootstrap protocol, and Appendix~\ref{app:error-taxonomy} gives the error-labeling decision order. 

\subsubsection*{AI Use Statement}
In this work, we used generative AI tools to generate the natural-language instances of SoftOR by applying a dataset-construction procedure that the authors designed without AI, to clean and reformat source datasets, and to assist in the exposition of author-developed adequacy proofs in Appendix~\ref{app:formal-adequacy}, but not with formulating the mathematical claims. It was also used to provide feedback on the experimental design, and to implement the ablation experiments. We have not used generative AI tools to develop the conceptual framework or the formal grammar, to propose or refine hypotheses, to design the methodology, or to interpret results. Translation and qualitative or thematic analysis are not applicable to this work. Additionally, we used generative AI tools to create and edit software code, to create scientific figures, to brainstorm, to edit the paper for readability, and to format references. We have reviewed all AI-assisted work. We take responsibility for the final content of this work, including text, claims, and artifacts produced with the aid of generative AI.

\clearpage
\bibliographystyle{plainnat}
\bibliography{references}

\clearpage
\appendix
\section{Additional Experimental Results}
\label{app:additional-results}

{\raggedbottom
% Canonical paper appendix maintained manually in the Overleaf repository.
\subsection{Hard-Constraint Accuracy}

All results in this appendix use the backbone, request budget, and bootstrap protocol of Section~\ref{sec:setup}.

\begin{table}[H]
\centering
\small
\caption{Solution accuracy (\%) on hard-constraint benchmarks. Bootstrap standard errors are in parentheses.}
\label{tab:rq1-hard-oneshot}
\resizebox{0.90\linewidth}{!}{%
\begin{tabular}{lrrrrr}
\toprule
Method & ComplexOR & IndustryOR & LogiOR & NLP4LP & Overall \\
\midrule
Direct & \textbf{94.4 (5.4)} & 75.9 (4.7) & 62.0 (5.0) & 78.1 (2.6) & 74.9 (2.1) \\
Codegen & 88.9 (7.5) & 75.9 (4.7) & 72.8 (4.7) & 87.2 (2.1) & 82.1 (1.8) \\
CoT & 83.3 (8.8) & 75.9 (4.7) & 72.8 (4.7) & 86.8 (2.2) & 81.6 (1.9) \\
ORThought & 83.3 (8.8) & 78.3 (4.5) & \textbf{73.9 (4.6)} & 85.5 (2.2) & 81.6 (1.9) \\
CoE & 88.9 (7.3) & \textbf{80.7 (4.3)} & 71.7 (4.7) & 86.4 (2.2) & \textbf{82.3 (1.8)} \\
ASTOR (Ours) & \textbf{94.4 (5.4)} & 77.1 (4.6) & 68.5 (4.9) & \textbf{87.6 (2.1)} & 81.8 (1.8) \\
\bottomrule
\end{tabular}}
\end{table}

\begin{table}[H]
\centering
\small
\caption{Paired comparison of ASTOR with CoE and Codegen on conventional hard-constraint benchmarks. $\Delta$ is ASTOR accuracy minus baseline accuracy (percentage points). Intervals are 95\% confidence intervals from 10,000 bootstrap resamples.}
\label{tab:paired-astor-hard}
\label{tab:paired-astor-coe}
\setlength{\tabcolsep}{5pt}
\renewcommand{\arraystretch}{1.15}
\begin{tabular}{@{}lrrrrrr@{}}
\toprule
 & \multicolumn{4}{c}{Instances solved correctly} & \multicolumn{2}{c}{ASTOR $-$ baseline} \\
\cmidrule(lr){2-5}\cmidrule(l){6-7}
Baseline & Both & ASTOR only & Baseline only & Neither & $\Delta$ (\%) & 95\% CI\\
\midrule
CoE & 345 & 11 & 13 & 66 & $-0.5$ & $[-2.5,\,1.8]$ \\
Codegen & 346 & 10 & 11 & 68 & $-0.2$ & $[-2.3,\,1.8]$ \\
\bottomrule
\end{tabular}
\end{table}

Table~\ref{tab:rq1-hard-oneshot} reports all six methods on the conventional hard-constraint benchmarks. Aggregated over the 435 instances, ASTOR is statistically indistinguishable from CoE (Table~\ref{tab:paired-astor-coe}), from Codegen (paired 95\% CI $[-2.3,1.8]$ pp), and from ORThought ($[-2.1,2.5]$ pp).

Table~\ref{tab:ablation-reasoning-hard} complements the soft-constraint backbone analysis in Table~\ref{tab:ablation-reasoning} by reporting the same comparison on the 55 hard-constraint instances in the frozen 100-instance ablation sample. 

\begin{table}[H]
\centering
\small
\caption{Model-backbone ablation on the 55 hard-constraint instances. Solution accuracy (\%); bootstrap standard errors are in parentheses.}
\label{tab:ablation-reasoning-hard}
\resizebox{0.8\linewidth}{!}{%
\begin{tabular}{lcccc}
\toprule
Method & \shortstack{Gemini 3.7 Flash\\(reasoning)} & \shortstack{Gemini 3.5 Flash Lite\\(no reasoning)} & \shortstack{GPT-5\\(reasoning, medium)} & \shortstack{Gemini 3.1 Pro\\(reasoning)} \\
\midrule
Direct & 78.2 (5.6) & 18.2 (5.2) & 78.2 (5.6) & 78.2 (5.6) \\
Codegen & 81.8 (5.2) & 78.2 (5.6) & 80.0 (5.4) & 80.0 (5.4) \\
CoT & 83.6 (5.0) & 76.4 (5.7) & 65.5 (6.4) & 80.0 (5.4) \\
ORThought & 83.6 (5.0) & \textbf{83.6 (5.0)} & \textbf{85.5 (4.8)} & 80.0 (5.4) \\
CoE & 83.6 (5.0) & 80.0 (5.4) & 81.8 (5.2) & \textbf{81.8 (5.2)} \\
ASTOR (Ours) & \textbf{85.5 (4.8)} & 78.2 (5.6) & 83.6 (5.0) & \textbf{81.8 (5.2)} \\
\bottomrule
\end{tabular}}
\end{table}

ASTOR leads on Gemini 3.7 Flash, ORThought leads on Gemini 3.5 Flash Lite and GPT-5, and ASTOR ties CoE for the lead on Gemini 3.1 Pro.

Table~\ref{tab:published-hard} lists results for four other LLM autoformulation methods as published by ORThought~\citep{orthought}. These methods differ from ours in inference budget, number of repetitions, and backbone: CoE and ORThought use GPT-4.1-nano, ORLM uses LLaMA-3-8B, and OptiMUS is reported only on NLP4LP using author-supplied intermediate JSON.

\begin{table}[H]
\centering
\small
\caption{Published hard-constraint success rates (\%) from~\citet{orthought}.}
\label{tab:published-hard}
\resizebox{0.6\linewidth}{!}{%
\begin{tabular}{lrrrr}
\toprule
Method & ComplexOR  & IndustryOR  & LogiOR  & NLP4LP  \\
\midrule
ORLM & 77.78 & 32.13 & 15.58 & 67.80 \\
OptiMUS & -- & -- & -- & 69.42 \\
CoE & 55.56 & 40.96 & 34.78 & 75.00 \\
ORThought & 77.78 & 57.83 & 46.01 & 89.02 \\
\bottomrule
\end{tabular}}
\end{table}

\IfFileExists{generated/reporting/rq2-hard-characteristics.pdf}{%
\begin{figure}[!htbp]
\centering
\includegraphics[width=0.98\linewidth]{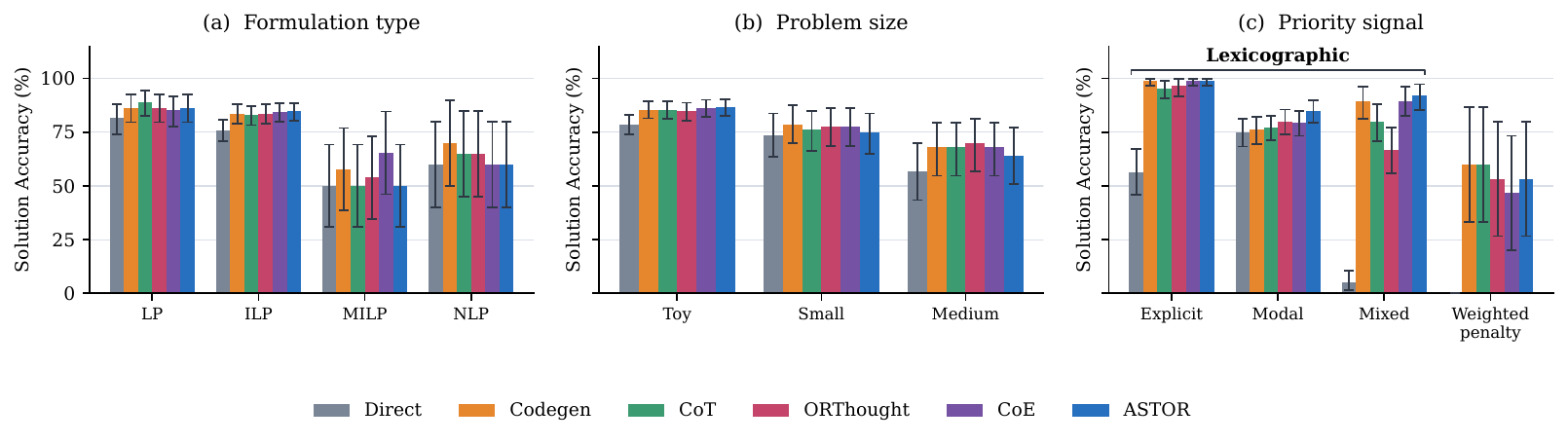}
\caption{Hard-constraint accuracy by (a) formulation type and (b) problem size, and (c) soft-constraint accuracy by priority signal. Whiskers show bootstrap 95\% confidence intervals.}
\label{fig:rq2-hard-characteristics}
\end{figure}
}{}

Figure~\ref{fig:rq2-hard-characteristics} disaggregates the hard-constraint benchmarks by formulation type and problem size for hard constraints, and priority signal for soft constraints. Problem types are Linear Program (LP), Integer Linear Program (ILP), Mixed Integer Linear Program (MILP), and Nonlinear Program (NLP). Problem sizes toy, small, and medium are defined in \citet{orthought}. Priority is signaled by explicit, modal, or mixed markers for lexicographic formulations, and weighted penalty formulations are always explicit. ASTOR leads on ILPs and on toy problems, and a different baseline leads in each remaining category, and achieves the highest accuracy for every lexicographic signal, with its largest margins when priorities are modal or mixed. Codegen and CoT tie for the highest accuracy on weighted-penalty Nurse Rostering.

\IfFileExists{generated/reporting/error-analysis-hard.pdf}{%
\begin{figure}[!htbp]
\centering
\includegraphics[width=0.5\linewidth]{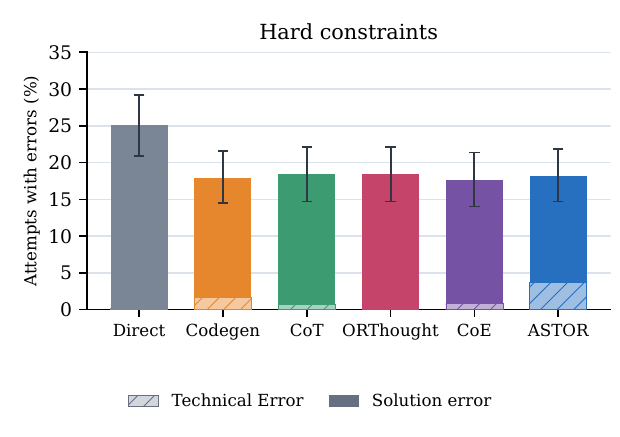}
\caption{Error rates on the hard-constraint benchmarks as percentages of all attempts. The unplotted remainder is correct. Whiskers show bootstrap 95\% confidence intervals for the total error rate.}
\label{fig:error-analysis}
\end{figure}
}{}

Figure~\ref{fig:error-analysis} reports error rates on the hard-constraint benchmarks. Every method except Direct has a similar total error rate.

\subsection{Constraint Satisfaction}

\begin{table}[H]
\centering
\small
\caption{Average hard constraint satisfaction (\%) over all instances on SoftOR.}
\vspace{0.2cm}
\label{tab:hard-satisfaction}
\begin{tabular}{lcccc}
\toprule
Method & \shortstack{Nurse\\Rostering} & COMPASS & \shortstack{Travel\\Planner} & LexiBench \\
\midrule
Direct & 31.8 & 57.6 & \textbf{99.0} & 5.3 \\
Codegen & \textbf{100.0} & \textbf{100.0} & 96.3 & \textbf{100.0} \\
CoT & \textbf{100.0} & 96.5 & 96.4 & 96.0 \\
ORThought & 86.7 & 97.6 & 97.4 & 78.7 \\
CoE & \textbf{100.0} & \textbf{100.0} & 97.5 & \textbf{100.0} \\
ASTOR (Ours) & \textbf{100.0} & \textbf{100.0} & 98.8 & \textbf{100.0} \\
\bottomrule
\end{tabular}
\end{table}

ASTOR satisfies every hard constraint on COMPASS and LexiBench and is within about a point of the lead on TravelPlanner; on Nurse Rostering it ties Codegen, CoT, and CoE at 100\% hard constraint satisfaction. More details on Nurse Rostering are available in Appendix~\ref{app:benchmark-construction}.

\subsection{Efficiency}

We report recorded monetary cost from provider token-usage telemetry and observed cache hits. Cost per correct solution divides the total recorded cost by the number of exact-correct predictions. The top row of Figure~\ref{fig:rq4-efficiency-cost} aggregates over all 790 evaluated problems and the bottom row restricts to the 355 SoftOR instances. Nurse Rostering costs use the new Gemini 3.8 Flash run; other benchmarks retain their original runs. Three Nurse Rostering API failures had no usage telemetry and are excluded from the recorded costs; budget reservations are not treated as measured charges.

\IfFileExists{generated/reporting/efficiency-cost-family.pdf}{%
\begin{figure}[H]
\centering
\includegraphics[width=0.78\linewidth]{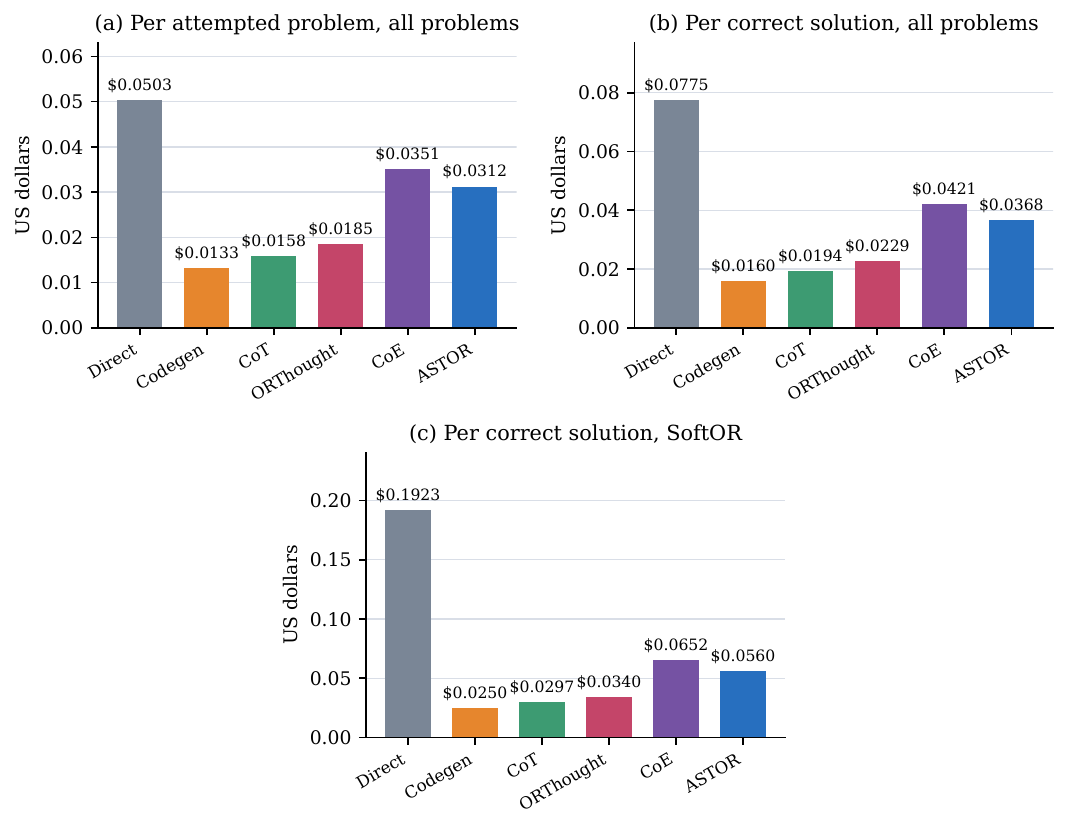}
\caption{Recorded monetary cost per attempted problem and per correct solution over (a, b) all evaluated problems and (c) SoftOR only. ASTOR uses provider-reported cached-input pricing.}
\label{fig:rq4-efficiency-cost}
\end{figure}
}{}

\subsection{Ablation Studies}

Figure~\ref{fig:ablation-components} repeats the component ablations of Section~\ref{sec:rq4} with the hard-constraint strata alongside the soft-constraint strata shown in the main text. Deterministic repair uses fixed rules in Python to fix local syntax errors in the model’s JSON output before parsing and compilation, without making another LLM call. Specifically, we use the JSON schema to correct misplaced closing braces and brackets and remove invalid labels on array elements, then check the repaired output’s structure and retention of the original strings and numbers before accepting it.

\IfFileExists{generated/reporting/ablation-components.pdf}{%
\begin{figure}[H]
\centering
\includegraphics[width=0.98\linewidth]{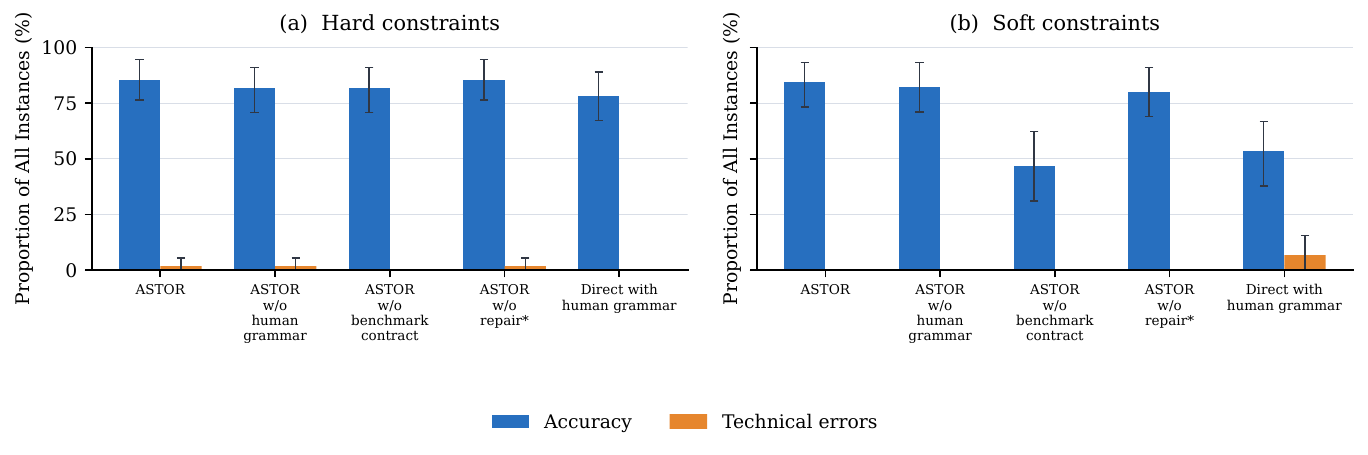}
\caption{ASTOR component ablation on (a) hard-constraint and (b) soft-constraint benchmarks. Blue bars indicate percent of formulations that ended in a correct solution, whereas orange bars indicate the percent of formulations that terminated in technical errors.   Whiskers show bootstrap 95\% confidence intervals.}
\label{fig:ablation-components}
\end{figure}
}{}

On the hard-constraint strata, removing the benchmark contract lowers accuracy by under four percentage points, and removing deterministic repair changes nothing. The Direct + AST control (the Direct prompt with the ASTOR grammar prepended as a reasoning aid; see the ablation definitions in Appendix~\ref{app:prompts}) matches Direct exactly on both families, 78.2\% on the hard-constraint instances and 53.3\% on the soft-constraint instances, so the grammar helps only when its output is compiled and solved.

}

\section{Worked Example}
\label{app:grammar}
\raggedbottom

We use the following natural-language instance throughout. Blue marks the hard
requirements and query, red marks the soft preferences.

\begin{tcolorbox}[
 enhanced,
 colback=pipegray!35,
 colframe=pipedark!30,
 boxrule=0.45pt,
 arc=1.5pt,
 left=8pt,right=8pt,top=7pt,bottom=7pt]
\textcolor{pipeblue}{A ward manager is scheduling Alice for an 8-hour Day shift
on any of Monday through Friday. Wednesday is already an agreed fixed day off,
and her contract requires her to work at least 16 hours this week.}
\textcolor{pipered}{She strongly prefers to work Monday and Tuesday, and
slightly prefers to work Thursday. The manager would like to minimize the
number of hours Alice works.}
\textcolor{pipeblue}{Which days should she work?}
\end{tcolorbox}

This appendix works the Nurse Rostering example through a compact view of the
grammar, from declarations to three finished models.

\subsection{Simplified Grammar View}

Figure~\ref{fig:grammar} summarizes the compositional rules of an example grammar constructed for this problem. The complete
grammar for ASTOR appears in Appendix~\ref{app:formal-adequacy}. A program is a block of
declarations followed by a model, and a model is one objective together with a
list of named constraints. Expressions are built from constants, symbols and
indexed symbols, combined with arithmetic operators, a small set of functions,
and aggregation over an index set. A formula either compares two expressions or
joins other formulas with the usual connectives. 

The objective is where the preferences are formalized, and it has three forms.
\texttt{Objective(feasibility)} is a problem with nothing to optimize.
\texttt{Objective(\textit{sense},$\,\langle$Expr$\rangle$)} is a single scalar
objective, which is the form a weighted-penalty model takes, since its penalty
terms are folded into that one expression.
\texttt{Objective(\textit{sense},$\,\langle$Tier$\rangle$,\ldots)} takes an
ordered list of tiers, each carrying its own expression, and a lexicographic
model uses that form instead. This lets an evaluator tell two formulations apart by looking at
the tree, without going back to the natural language.

\begin{figure}[H]
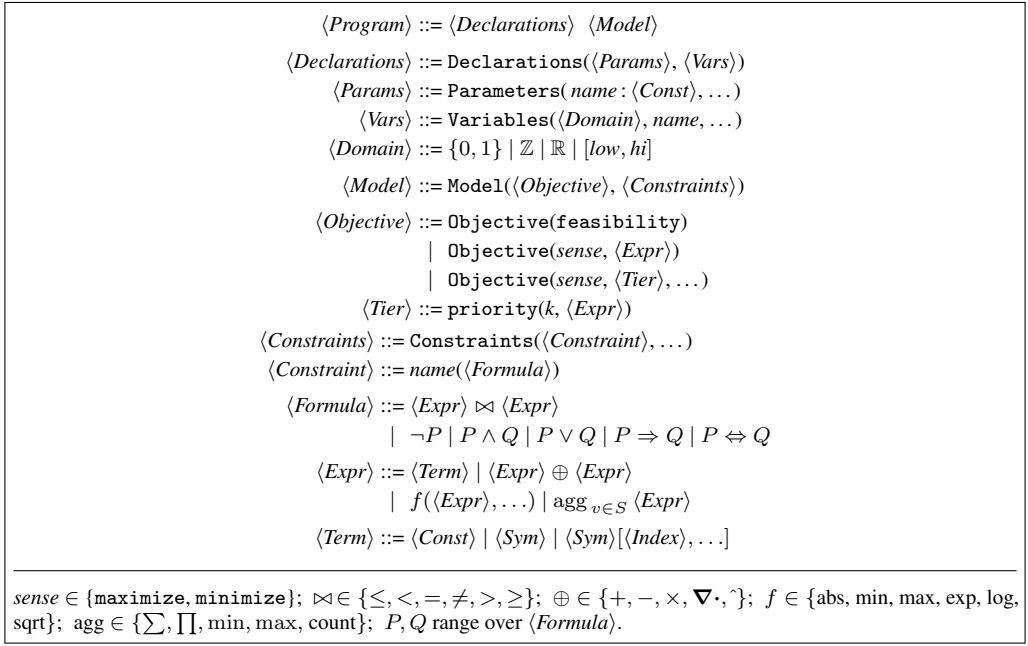

\centering
\fcolorbox{optboxcolor}{white}{%
\begin{minipage}{0.95\linewidth}
\centering
\footnotesize
\setlength{\tabcolsep}{1.5pt}
\renewcommand{\arraystretch}{1.15}
\begin{tabular}[t]{@{}r@{\ }c@{\ }l@{}}
$\langle$\textit{Program}$\rangle$ & ::= &
 $\langle$\textit{Declarations}$\rangle$\; $\langle$\textit{Model}$\rangle$\\[1mm]

$\langle$\textit{Declarations}$\rangle$ & ::= &
 \texttt{Declarations}($\langle$\textit{Params}$\rangle$, $\langle$\textit{Vars}$\rangle$)\\
$\langle$\textit{Params}$\rangle$ & ::= &
 \texttt{Parameters}(\,\textit{name}\,:\,$\langle$\textit{Const}$\rangle$, \ldots)\\
$\langle$\textit{Vars}$\rangle$ & ::= &
 \texttt{Variables}($\langle$\textit{Domain}$\rangle$, \textit{name}, \ldots)\\
$\langle$\textit{Domain}$\rangle$ & ::= &
 $\{0,1\}$ $\mid$ $\mathbb{Z}$ $\mid$ $\mathbb{R}$ $\mid$ $[\textit{low},\textit{hi}]$\\[1mm]

$\langle$\textit{Model}$\rangle$ & ::= &
 \texttt{Model}($\langle$\textit{Objective}$\rangle$, $\langle$\textit{Constraints}$\rangle$)\\[1mm]

$\langle$\textit{Objective}$\rangle$ & ::= &
 \texttt{Objective}(\texttt{feasibility})\\
 & $\mid$ & \texttt{Objective}(\textit{sense}, $\langle$\textit{Expr}$\rangle$)\\
 & $\mid$ & \texttt{Objective}(\textit{sense}, $\langle$\textit{Tier}$\rangle$, \ldots)\\
$\langle$\textit{Tier}$\rangle$ & ::= &
 \texttt{priority}(\textit{k}, $\langle$\textit{Expr}$\rangle$)\\
\end{tabular}
\hspace{7mm}
\begin{tabular}[t]{@{}r@{\ }c@{\ }l@{}}
$\langle$\textit{Constraints}$\rangle$ & ::= &
 \texttt{Constraints}($\langle$\textit{Constraint}$\rangle$, \ldots)\\
$\langle$\textit{Constraint}$\rangle$ & ::= &
 \textit{name}($\langle$\textit{Formula}$\rangle$)\\[1mm]

$\langle$\textit{Formula}$\rangle$ & ::= &
 $\langle$\textit{Expr}$\rangle \bowtie \langle$\textit{Expr}$\rangle$\\
 & $\mid$ & $\neg P$ $\mid$ $P \wedge Q$ $\mid$ $P \vee Q$
 $\mid$ $P \Rightarrow Q$ $\mid$ $P \Leftrightarrow Q$\\[1mm]

$\langle$\textit{Expr}$\rangle$ & ::= &
 $\langle$\textit{Term}$\rangle$ $\mid$
 $\langle$\textit{Expr}$\rangle \oplus \langle$\textit{Expr}$\rangle$\\
 & $\mid$ & $f(\langle$\textit{Expr}$\rangle,\ldots)$ $\mid$
 $\mathrm{agg}_{\,v \in S}\, \langle$\textit{Expr}$\rangle$\\[1mm]

$\langle$\textit{Term}$\rangle$ & ::= &
 $\langle$\textit{Const}$\rangle$ $\mid$
 $\langle$\textit{Sym}$\rangle$ $\mid$
 $\langle$\textit{Sym}$\rangle[\langle$\textit{Index}$\rangle,\ldots]$\\
\end{tabular}

\vspace{2mm}
\rule{\linewidth}{0.4pt}

\vspace{1mm}
{\raggedright\footnotesize
\textit{sense} $\in$ \{\texttt{maximize}, \texttt{minimize}\};\;
$\bowtie\, \in \{\le,<,=,\neq,>,\ge\}$;\;
$\oplus \in \{+,-,\times,\div,\hat{\ }\}$;\;
$f \in \{$abs, min, max, exp, log, sqrt$\}$;\;
agg $\in \{\sum,\prod,\min,\max,$ count$\}$;\;
$P, Q$ range over $\langle$\textit{Formula}$\rangle$.\par}
\end{minipage}}
\caption{An example grammar, which generates optimization problems that have hard constraints, weighted-penalty objectives, or lexicographic preferences. Preferences are syntactically encoded. }
\label{fig:grammar}
\end{figure}

\subsection{Declarations}

Figure~\ref{fig:decl} shows the declarations for the example. Parameters hold
the numbers taken from the instance, here the shift length and the contract
minimum. Variables hold the decisions, one binary per day of the planning
period. 

\begin{figure}[H]
\centering
\resizebox{0.92\textwidth}{!}{%
\begin{tikzpicture}[x=1mm,y=1mm,font=\sffamily,
 nt/.style={draw=pipedark,line width=.6pt,rounded corners=1.8pt,fill=white,
 inner sep=3pt,font=\footnotesize\bfseries,text=pipedark},
 lf/.style={draw=pipedark!55,line width=.5pt,rounded corners=1.8pt,fill=pipegray,
 inner sep=3pt,font=\footnotesize,text=pipedark,align=center},
 te/.style={draw=pipedark!60,line width=.5pt}]

\node[nt] (dec) at (83,0) {Declarations};
\node[nt] (par) at (42,-14) {Parameters};
\node[nt] (var) at (124,-14) {Variables};
\foreach \c in {par,var}{\draw[te] (dec.south) -- ++(0,-5mm) -| (\c.north);}

\node[lf,anchor=north] (a1) at (28,-25) {shift\_hours\\\textcolor{pipeblue}{8}};
\node[lf,anchor=north] (a2) at (56,-25) {min\_hours\\\textcolor{pipeblue}{16}};
\foreach \c in {a1,a2}{\draw[te] (par.south) -- ++(0,-4mm) -| (\c.north);}

\node[lf] (dom) at (124,-23) {$\in\{0,1\}$};
\draw[te] (var.south) -- (dom.north);
\node[lf,anchor=north] (b1) at ( 88,-33) {$x_{\mathrm{Mon}}$};
\node[lf,anchor=north] (b2) at (106,-33) {$x_{\mathrm{Tue}}$};
\node[lf,anchor=north] (b3) at (124,-33) {$x_{\mathrm{Wed}}$};
\node[lf,anchor=north] (b4) at (142,-33) {$x_{\mathrm{Thu}}$};
\node[lf,anchor=north] (b5) at (160,-33) {$x_{\mathrm{Fri}}$};
\foreach \c in {b1,b2,b3,b4,b5}{\draw[te] (dom.south) -- ++(0,-4mm) -| (\c.north);}

\end{tikzpicture}}
\caption{Declaration subtree for the nurse example, shared by all three models of
Figure~\ref{fig:compare}. Declares numeric data and variable names to be used in AST construction.}
\label{fig:decl}
\end{figure}
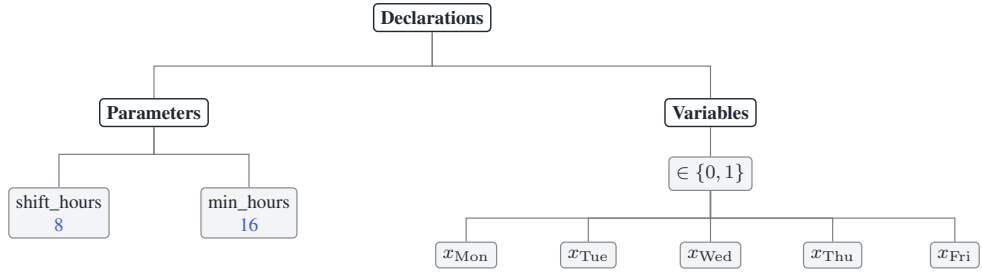

\subsection{Three Formulations}

Figure~\ref{fig:compare} sets the three formulations side by side, each AST
above the model it compiles to. The declarations and the constraint subtree are
the same in all three, and only the objective branch differs.

\begin{figure}[H]
\centering
% Compact re-rendering of the three ASTs so they sit side by side, each above
% the optimization model it compiles to. Column bounding boxes are pinned with
% \path so the three trees, arrows and models stay row-aligned.
\footnotesize
$\hrs = 8\,(x_{\mathrm{Mon}}+x_{\mathrm{Tue}}+x_{\mathrm{Thu}}+x_{\mathrm{Fri}})$,
the total hours worked.
\par\medskip
\setlength{\tabcolsep}{0pt}
\resizebox{\linewidth}{!}{%
\begin{tabular}{@{}c@{\hspace{3mm}}c@{\hspace{3mm}}c@{}}
\textbf{\footnotesize (a) Hard constraints} &
\textbf{\footnotesize (b) Weighted penalty} &
\textbf{\footnotesize (c) Lexicographic} \\[1.5mm]

% ================= ASTs =================
\begin{tikzpicture}[
 it/.style={anchor=north west, inner sep=0pt, font=\scriptsize, text=pipedark},
 hd/.style={it, font=\scriptsize\bfseries},
 astedge/.style={draw=pipedark!50, line width=.4pt}]
\path (0,2mm) (46mm,-52mm);
\node[hd] (m) at (0,0) {Model};
\node[hd] (o) at ([shift={(3mm,-1.7mm)}]m.south west) {Objective:};
\node[it] (o1) at ([shift={(3mm,-1.7mm)}]o.south west) {minimize};
\node[it] (o2) at ([shift={(3mm,-1.7mm)}]o1.south west) {$\hrs$};
\astbranch{m}{o}\astbranch{o}{o1}\astbranch{o1}{o2}
\node[hd] (c) at ([shift={(0,-2.6mm)}]o.south west |- o2.south west) {Constraints:};
\node[it] (c1) at ([shift={(3mm,-1.7mm)}]c.south west) {min\_hours};
\node[it] (c2) at ([shift={(3mm,-1.7mm)}]c1.south west) {$\hrs \geq 16$};
\node[it] (c3) at ([shift={(-3mm,-1.7mm)}]c2.south west){fixed\_day\_off};
\node[it] (c4) at ([shift={(3mm,-1.7mm)}]c3.south west) {$x_{\mathrm{Wed}} = 0$};
\astbranch{m}{c}\astbranch{c}{c1}\astbranch{c1}{c2}\astbranch{c}{c3}\astbranch{c3}{c4}
% arrow starts under THIS column's last node and ends on the shared baseline,
% so the leftover height is taken up by the arrow instead of blank space
\coordinate (tip) at (20mm,0 |- c4.south);
\draw[-{Latex[length=2.4mm,width=2mm]},line width=1pt,draw=pipeblue]
 ([yshift=-3.5mm]tip) -- (20mm,-50mm);
\end{tikzpicture}
&
\begin{tikzpicture}[
 it/.style={anchor=north west, inner sep=0pt, font=\scriptsize, text=pipedark},
 hd/.style={it, font=\scriptsize\bfseries},
 astedge/.style={draw=pipedark!50, line width=.4pt}]
\path (0,2mm) (46mm,-52mm);
\node[hd] (m) at (0,0) {Model};
\node[hd] (o) at ([shift={(3mm,-1.7mm)}]m.south west) {Objective:};
\node[it] (o1) at ([shift={(3mm,-1.7mm)}]o.south west) {minimize};
\node[it] (o2) at ([shift={(3mm,-1.7mm)}]o1.south west) {$\hrs + \lambda(1-x_{\mathrm{Mon}})$};
\node[it] (o3) at ([shift={(0,-1.3mm)}]o2.south west)
 {$\quad+\,\lambda(1-x_{\mathrm{Tue}})+\mu(1-x_{\mathrm{Thu}})$};
\astbranch{m}{o}\astbranch{o}{o1}\astbranch{o1}{o2}
\node[hd] (c) at ([shift={(0,-2.6mm)}]o.south west |- o3.south west) {Constraints:};
\node[it] (c1) at ([shift={(3mm,-1.7mm)}]c.south west) {min\_hours};
\node[it] (c2) at ([shift={(3mm,-1.7mm)}]c1.south west) {$\hrs \geq 16$};
\node[it] (c3) at ([shift={(-3mm,-1.7mm)}]c2.south west){fixed\_day\_off};
\node[it] (c4) at ([shift={(3mm,-1.7mm)}]c3.south west) {$x_{\mathrm{Wed}} = 0$};
\astbranch{m}{c}\astbranch{c}{c1}\astbranch{c1}{c2}\astbranch{c}{c3}\astbranch{c3}{c4}
% arrow starts under THIS column's last node and ends on the shared baseline,
% so the leftover height is taken up by the arrow instead of blank space
\coordinate (tip) at (20mm,0 |- c4.south);
\draw[-{Latex[length=2.4mm,width=2mm]},line width=1pt,draw=pipeblue]
 ([yshift=-3.5mm]tip) -- (20mm,-50mm);
\end{tikzpicture}
&
\begin{tikzpicture}[
 it/.style={anchor=north west, inner sep=0pt, font=\scriptsize, text=pipedark},
 hd/.style={it, font=\scriptsize\bfseries},
 pr/.style={it, font=\scriptsize\bfseries, text=pipered},
 astedge/.style={draw=pipedark!50, line width=.4pt}]
\path (0,2mm) (46mm,-52mm);
\node[hd] (m) at (0,0) {Model};
\node[hd] (o) at ([shift={(3mm,-1.7mm)}]m.south west) {Objective:};
\node[it] (o1) at ([shift={(3mm,-1.7mm)}]o.south west) {minimize};
\node[it] (t1) at ([shift={(3mm,-1.7mm)}]o1.south west)
 {\textcolor{pipered}{\bfseries priority 1}\; $(1-x_{\mathrm{Mon}})+(1-x_{\mathrm{Tue}})$};
\node[it] (t2) at ([shift={(0,-1.7mm)}]t1.south west)
 {\textcolor{pipered}{\bfseries priority 2}\; $(1-x_{\mathrm{Thu}})$};
\node[it] (t3) at ([shift={(0,-1.7mm)}]t2.south west)
 {\textcolor{pipered}{\bfseries priority 3}\; $\hrs$};
\astbranch{m}{o}\astbranch{o}{o1}
\astbranch{o1}{t1}\astbranch{o1}{t2}\astbranch{o1}{t3}
\node[hd] (c) at ([shift={(0,-2.6mm)}]o.south west |- t3.south west) {Constraints:};
\node[it] (c1) at ([shift={(3mm,-1.7mm)}]c.south west) {min\_hours};
\node[it] (c2) at ([shift={(3mm,-1.7mm)}]c1.south west) {$\hrs \geq 16$};
\node[it] (c3) at ([shift={(-3mm,-1.7mm)}]c2.south west){fixed\_day\_off};
\node[it] (c4) at ([shift={(3mm,-1.7mm)}]c3.south west) {$x_{\mathrm{Wed}} = 0$};
\astbranch{m}{c}\astbranch{c}{c1}\astbranch{c1}{c2}\astbranch{c}{c3}\astbranch{c3}{c4}
% arrow starts under THIS column's last node and ends on the shared baseline,
% so the leftover height is taken up by the arrow instead of blank space
\coordinate (tip) at (20mm,0 |- c4.south);
\draw[-{Latex[length=2.4mm,width=2mm]},line width=1pt,draw=pipeblue]
 ([yshift=-3.5mm]tip) -- (20mm,-50mm);
\end{tikzpicture}
\\[1mm]

% ================= models =================
\fcolorbox{optboxcolor}{white}{\begin{minipage}[t][26mm][t]{47mm}\vspace{1mm}
\centering\scriptsize$\begin{aligned}[t]
\min\quad & \hrs \\
\text{s.t.}\quad & \hrs \geq 16 \\
 & x_{\mathrm{Wed}} = 0 \\
 & \mathbf{x} \in \{0,1\}^5
\end{aligned}$\vspace{1mm}\end{minipage}} &
\fcolorbox{optboxcolor}{white}{\begin{minipage}[t][26mm][t]{47mm}\vspace{1mm}
\centering\scriptsize$\begin{aligned}[t]
\min\quad & \hrs \\
 & +\,\lambda(1-x_{\mathrm{Mon}})+\lambda(1-x_{\mathrm{Tue}}) \\
 & +\,\mu(1-x_{\mathrm{Thu}}) \\
\text{s.t.}\quad & \hrs \geq 16 \\
 & x_{\mathrm{Wed}} = 0 \\
 & \mathbf{x} \in \{0,1\}^5
\end{aligned}$\vspace{1mm}\end{minipage}} &
\fcolorbox{optboxcolor}{white}{\begin{minipage}[t][26mm][t]{47mm}\vspace{1mm}
\centering\scriptsize$\begin{aligned}[t]
\operatorname*{lex\,min}\big( & (1-x_{\mathrm{Mon}})+(1-x_{\mathrm{Tue}}), \\
 & (1-x_{\mathrm{Thu}}), \\
 & \hrs \big) \\
\text{s.t.}\quad & \hrs \geq 16 \\
 & x_{\mathrm{Wed}} = 0 \\
 & \mathbf{x} \in \{0,1\}^5
\end{aligned}$\vspace{1mm}\end{minipage}} \\[1.5mm]

% ================= gold row =================
\scriptsize\textit{optimum:} any two shifts \ (16\,h) &
\scriptsize\textit{optimum:} Mon, Tue \ (16\,h) &
\scriptsize\textit{optimum:} Mon, Tue, Thu \ (24\,h) \\
\end{tabular}}

\caption{Different formulations of the nurse problem. Formulation (a) contains no preferences, (b) penalizes violations by scalar terms $\lambda, \mu$ chosen by an author, and (c) tiers the preferences based on lexicographic strength, with the time objective as the final tie-breaker.}
\label{fig:compare}
\end{figure}
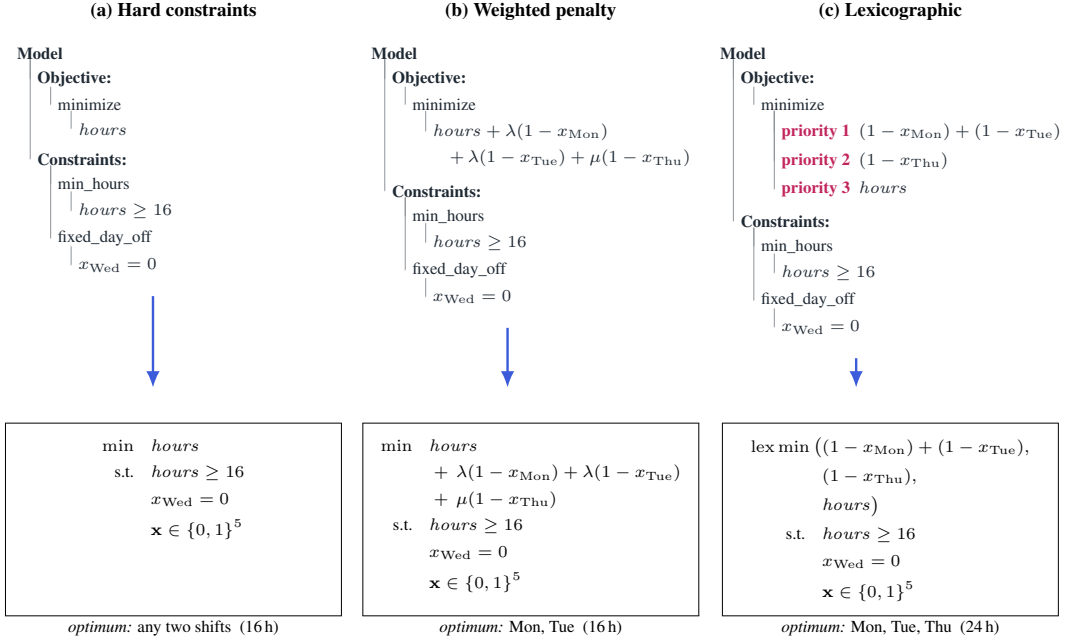

The optima differ as well. With no preferences stated, any two shifts that
avoid Wednesday meet the 16-hour minimum, so the answer is not unique. Under
the weighted objective at $\lambda=3$ and $\mu=1$, Alice works Monday and
Tuesday: satisfying the Thursday request would cost eight more hours to relieve
one point of penalty, so the solver drops it. Under the lexicographic objective
she works Monday, Tuesday and Thursday, because hours only break ties among
rosters that already satisfy both request levels, and no saving in hours can
outrank a request at a higher level.

A weighted model will trade a preference away once the objective pays enough for
it, and a lexicographic model will not.

\subsection{Literal Compiler Trace}

\begin{tcolorbox}[
 enhanced,
 colback=pipegray!35,
 colframe=pipedark!30,
 colbacktitle=pipelight,
 coltitle=pipedark,
 fonttitle=\sffamily\bfseries\small,
 title={LexiBench: habitat crew-days},
 boxrule=0.45pt,
 arc=1.5pt,
 left=8pt,right=8pt,top=7pt,bottom=7pt]
\textbf{Problem statement.}
Divide the season's habitat-restoration crew-days among erosion control, wetland
repair, native planting, and trail improvement. The wish we would surrender
last is that at least 21 crew-days go jointly to erosion control and wetland
repair; the others together do not equal it. The season allows 29 crew-days in
total, and weather rules out any extension. Ignore how the days were split last
season. Those remaining wishes are tied with one another and rank below it:
every one of the 29 available crew-days is put to use; at least 10 crew-days go
to native planting; and at least 10 crew-days go to trail improvement. Each
site must receive at least one crew-day. Allocations must be nonnegative whole
numbers. Determine the allocation.

\medskip
\textbf{Instance data.}
\begin{center}
\small
\renewcommand{\arraystretch}{1.12}
\begin{tabular}{lr}
\hline
Restoration category & Unit cost \\
\hline
Erosion control & 13 \\
Wetland repair & 10 \\
Native planting & 12 \\
Trail improvement & 9 \\
\hline
\end{tabular}
\qquad
\begin{tabular}{lr}
\hline
Available crew-days & 29 \\
Minimum per category & 1 \\
\hline
\end{tabular}
\end{center}
\textbf{Benchmark stratum.}
Constructive; direct priority dominance; easy.
\end{tcolorbox}

The model received this statement with a value-free data contract and returned
an AST serialized as JSON. The reader view below is rendered directly from that
JSON and preserves every declaration, hard constraint, soft constraint,
priority, preference relation, objective, and unresolved item. It suppresses
only serialization scaffolding such as repeated node-type labels, null defaults,
expression wrappers, descriptions, and copied source text. The JSON Schema and
benchmark instance are available in the anonymous repository. The instance data
were bound only after formulation.

\tcbinputlisting{
 enhanced,
 breakable,
 listing only,
 listing file={appendix_lexibench_habitat_ast_reader.txt},
 listing options={style=appendixprompt},
 colback=pipegray!35,
 colframe=pipedark!30,
 colbacktitle=pipelight,
 coltitle=pipedark,
 fonttitle=\sffamily\bfseries\small,
 title={Reader view of the model-produced AST for LexiBench habitat crew-days},
 title after break={Reader view of the model-produced AST (continued)},
 boxrule=0.45pt,
 arc=1.5pt,
 left=6pt,
 right=6pt,
 top=5pt,
 bottom=5pt,
 before skip=6pt,
 after skip=8pt
}

\begin{tcolorbox}[
 enhanced,
 colback=pipegray!35,
 colframe=pipedark!30,
 colbacktitle=pipelight,
 coltitle=pipedark,
 fonttitle=\sffamily\bfseries\small,
 title={Deterministic continuation},
 boxrule=0.45pt,
 arc=1.5pt,
 left=8pt,right=8pt,top=7pt,bottom=7pt]
Let $x_E,x_W,x_N,x_T$ denote crew-days assigned to erosion control, wetland
repair, native planting, and trail improvement. The compiler produced
\[
x_c\in\mathbb Z_{\geq 0},\qquad
\sum_c x_c\leq 29,\qquad x_c\geq 1\quad(c\in\{E,W,N,T\}),
\]
with the lexicographic satisfaction vector
\[
\operatorname{lexmax}\left(
 \mathbbm 1[x_E+x_W\geq21],\
 \mathbbm 1\!\left[\sum_c x_c=29\right]
 +\mathbbm 1[x_N\geq10]+\mathbbm 1[x_T\geq10]
\right).
\]
Gurobi returned $(x_E,x_W,x_N,x_T)=(1,26,1,1)$. Its tier vector $(1,1)$
matches the gold tier vector exactly.
\end{tcolorbox}
\flushbottom
\section{Formal Grammar and Its Adequacy}
\label{app:formal-adequacy}

\subsection{Complete Grammar}

The complete abstract-syntax grammar used by ASTOR is given here.
The suffixes \texttt{X*}, \texttt{X?}, \texttt{X\{n\}}, and
\texttt{X\{n,\}} denote, respectively, zero or more occurrences, an optional
occurrence, exactly $n$ occurrences, and at least $n$ occurrences.
We read $G$ as a typed grammar: symbol occurrences resolve to declarations,
operators are applied to arguments of the displayed sorts, penalty weights and
violations are nonnegative, and preference relations determine ordered
priority tiers, allowing ties.

\noindent\begin{minipage}[t]{0.485\linewidth}
\begin{grammarblock}{Program and data declarations}
Program
 declarations : Declaration*
 objective : Objective

Declaration ::=
 Set | Parameter | Variable | Piecewise
 | Constraint | SoftConstraint | PreferenceRelation

Set
 name : Name
 elements : SetValue*?

Parameter
 name : Name
 domain : ParameterDomain?
 indexed_by : Name*
 value_ast : Constant?
 values : ValueTable?
 units : Text?
\end{grammarblock}
\end{minipage}\hfill
\begin{minipage}[t]{0.485\linewidth}
\begin{grammarblock}{Decision declarations}

Variable
 name : Name
 domain : Domain
 indexed_by : Name*
 integrality_basis : Basis?
 units : Text?
 lower_bound : Expr?
 upper_bound : Expr?
 categories : SetValue*?

Piecewise
 name : Name
 input : Expr
 breakpoints : Number{2,}
 values : Number*?
 marginal_rates : Number*?
\end{grammarblock}
\end{minipage}

\begin{grammarblock}{Constraints, preferences, and objectives}
Constraint
 name : Name
 body : Formula
 source_text : Text?

SoftConstraint
 name : Name
 condition : Formula
 guard : Formula?
 aggregation : Aggregation?
 weight : Constant?
 priority : Constant?
 violation : Expr?
 source_text : Text?

PreferenceRelation
 relation : Relation
 left : Name | Expr
 right : Name | Expr
 source_text : Text?

Objective
 sense : minimize | maximize | feasibility
 expression : Expr?
\end{grammarblock}

\begin{grammarblock}{Decision-dependent formulas and expressions}
Formula ::=
 comparison (operator: CompareOp, left: Expr, right: Expr)
 | logical (operator: and|or, arguments: Formula{2,})
 | logical (operator: not, arguments: Formula{1})
 | logical (operator: implies|iff|xor, arguments: Formula{2})
 | quantifier (quantifier: forall|exists, variable: Name, domain: Name,
 condition: DataFormula?, body: Formula)
 | membership (element: DataExpr, collection: DataExpr, negated: Bool?)

Expr ::=
 constant (value: Value, value_type: ValueType?)
 | symbol (name: Name)
 | indexed_symbol (name: Name, indices: Expr*)
 | operation (operator: Op, arguments: Expr*,
 order: NormOrder?)
 | operation (operator: if,
 arguments: (Formula, Expr, Expr))
 | aggregation (operator: sum|product|min|max|count|average,
 index: Name, domain: Name,
 condition: DataFormula?, expression: Expr)
\end{grammarblock}

\begin{grammarblock}{Data-only formulas and expressions}
DataFormula ::=
 comparison (operator: CompareOp, left: DataExpr, right: DataExpr)
 | logical (operator: and|or, arguments: DataFormula{2,})
 | logical (operator: not, arguments: DataFormula{1})
 | logical (operator: implies|iff|xor, arguments: DataFormula{2})
 | membership (element: DataExpr, collection: DataExpr, negated: Bool?)

DataExpr ::=
 constant (value: Value, value_type: ValueType?)
 | symbol (name: DataName)
 | indexed_symbol (name: ParameterName, indices: DataExpr*)
 | operation (operator: Op, arguments: DataExpr*,
 order: NormOrder?)
 | operation (operator: if,
 arguments: (DataFormula, DataExpr, DataExpr))
 | aggregation (operator: sum|product|min|max|count|average,
 index: Name, domain: Name,
 condition: DataFormula?, expression: DataExpr)
\end{grammarblock}

\begin{grammarblock}{Domains and terminal symbols}
Domain ::= binary | integer | nonnegative_integer | real
 | nonnegative_real | bounded_integer | bounded_real
 | categorical
ParameterDomain ::= integer | real | boolean | categorical
Basis ::= explicitly_stated_lp | explicitly_stated_integer
 | indivisible_unit | divisible_quantity | convention_default
Aggregation ::= lexicographic | weighted_sum
Relation ::= prefer | higher_priority
CompareOp ::= le | ge | eq | lt | gt | neq
Op ::= add | subtract | multiply | divide | negate | power
 | sqrt | exp | log | sin | cos | tan | abs
 | min | max | norm
NormOrder ::= 0 | 1 | 2 | inf
ValueType ::= integer | real | rational | boolean | string
Value ::= Number | Text | Bool
Constant ::= constant (value: Value, value_type: ValueType?)
ValueTable ::= { Text : Value }
SetValue ::= Number | Text
DataName ::= SetName | ParameterName | IndexName
SetName ::= Name
ParameterName ::= Name
IndexName ::= Name
Name ::= identifier
\end{grammarblock}

\subsection{Soft-constraint Specifications}
\paragraph{Soft-constraint semantics.}
Each soft constraint has a \texttt{condition}: a true-or-false
statement describing the desired outcome, such as ``the hotel
costs at most \$100.'' An optional \texttt{guard} is another
true-or-false statement that determines whether the preference
applies, such as ``the hotel is selected.'' Without a guard,
the preference always applies.

For a solution $x$, the penalty $v_i(x)$ is zero if the guard
is false. Otherwise, it is zero if the condition is true and
one if it is false. In \texttt{weighted\_sum} mode, an optional
\texttt{violation} expression can specify the penalty instead.
This expression must be nonnegative and zero exactly when
the condition is true.

All soft constraints must specify the same aggregation mode.
For \texttt{weighted\_sum}, each must supply a nonnegative weight
that multiplies its penalty. Priorities are unused.
For \texttt{lexicographic}, the objective counts violated
preferences within each tier, with each violated preference
contributing one regardless of its weight or violation expression.
Tiers with larger priorities are optimized first.
Every constraint's tier must be specified explicitly or fully
determined by preference relations. No default priority is assumed.

\subsection{Representational Adequacy}

Let $\mathcal O$ be the class of finite optimization problems over binary, integer, real, bounded, or finite categorical decision domains. Hard constraints and objectives are formed by finite compositions of constants, indexed quantities, arithmetic and elementary functions, norms, conditionals, piecewise-linear functions, finite aggregations, comparisons, Boolean connectives, and quantification over finite index sets. A problem may additionally contain nonnegative soft-constraint violations combined by either weighted penalties or a finite lexicographic priority ordering.

A \emph{sort} is the type of value represented by a node, such as an integer, real number, Boolean value, string, or element of a declared finite set. An AST is \emph{well-sorted} when every operation is applied to inputs of the appropriate types. For example, arithmetic operations require numeric inputs, a comparison of two numeric expressions produces a Boolean-valued condition, and a quantified index must range over a declared finite set. An AST is \textit{well-formed} if its construction does not contradict the rules of $G$ and if all references and indices are valid, expressions and aggregations are defined wherever evaluated, piecewise constructions are unambiguous, and priority annotations and relations determine ordered tiers for all
soft constraints, allowing ties. A partial ordering that leaves a
constraint's tier unresolved is insufficient.

Let $\mathcal A$ be the set of well-formed, well-sorted ASTs generated by $G$. The grammar defines a map
\[
\llbracket \cdot \rrbracket :
\mathcal A \longrightarrow \mathcal{O}.
\]
For an AST $A$, $\llbracket A \rrbracket$ is the optimization problem described by $A$: its decision domain, feasible set, and ordering of feasible solutions. This map is evaluated recursively, with each node applying its specified operation to the meanings of its children.

We consider two optimization problems equivalent if they have the same decision domain and feasible set and rank all feasible solutions in the same order.
For this representation, adequacy means that the semantic image of $G$ is exactly $\mathcal O$: the grammar represents all and only the problems in the target class \citep{adequacy3,model_completeness}. The following results prove the two set inclusions.

\begin{lemma}[Representation of expressions and conditions]
\label{lem:term-formula-representation}
Every \texttt{Expr} generated by $G$ is a well-sorted term, and every \texttt{Formula} generated by $G$ evaluates to either true or false. Conversely, every expression and condition allowed in a problem in $\mathcal O$ can be represented by an \texttt{Expr} or \texttt{Formula} in $G$. The original and represented versions return the same value or truth value for every valid assignment of the decision variables and data.
\end{lemma}

\begin{proof}
For the forward direction, we use induction over the nodes of a well-formed AST generated by $G$, proceeding from the leaves to the root.

\paragraph{Base case.}
A leaf is a \texttt{constant}, \texttt{symbol}, or \texttt{indexed\_symbol}. A \texttt{constant} has its stated value and sort; a \texttt{symbol} or \texttt{indexed\_symbol} inherits the sort of its declaration. Since $G$ requires symbols and indices to agree with their declarations, every leaf is well-sorted.

\paragraph{Inductive step.}
The inductive hypothesis is that every child \texttt{Expr} or \texttt{DataExpr} is well-sorted and every child \texttt{Formula} or \texttt{DataFormula} is Boolean-valued. Proceeding to the parent, an \texttt{operation} applies a typed operator to compatible expressions; a conditional selects between compatible branches using a Boolean condition; and an \texttt{aggregation} combines well-sorted values over a declared finite domain. A \texttt{comparison}, \texttt{membership}, \texttt{logical}, or \texttt{quantifier} node returns a Boolean value when applied to children of the required sorts. By the inductive hypothesis and the typing requirements of $G$, the parent is therefore well-sorted if it is an expression and Boolean-valued if it is a formula. This completes the induction.

For the converse direction, we use induction over the nodes of any expression or condition allowed in a problem in $\mathcal O$.

\paragraph{Base case.}
A leaf is a constant, scalar quantity, or indexed quantity. It is represented by a \texttt{constant}, \texttt{symbol}, or \texttt{indexed\_symbol} node, respectively, with the same value and sort.

\paragraph{Inductive step.}
The inductive hypothesis is that every child has a representation in $G$ returning the same value or truth value on every valid assignment. Proceeding to the parent, arithmetic and elementary functions, norms, and conditionals are represented by \texttt{operation}; finite aggregations by \texttt{aggregation}; piecewise-linear expressions by \texttt{Piecewise}; comparisons and membership conditions by their corresponding \texttt{Formula} productions; Boolean combinations by \texttt{logical}; and finite universal or existential conditions by \texttt{quantifier}. Data-only constructions use the corresponding \texttt{DataExpr} and \texttt{DataFormula} productions. By the inductive hypothesis, the represented children agree with the original children; because the parent applies the same operation, it also returns the same value or truth value. This completes the induction.
\end{proof}

\begin{theorem}[Well-definedness]
\label{thm:grammar-well-definedness}
For every $A\in\mathcal A$,
\[
\llbracket A\rrbracket\in\mathcal O.
\]
\end{theorem}

\begin{proof}
Fix $A\in\mathcal A$. We verify each condition in the definition of $\mathcal O$.

\paragraph{Decision domain.}
The \texttt{Program} production contains finitely many declarations and one \texttt{Objective}. Its \texttt{Set} declarations define finite index sets, and its \texttt{Parameter} declarations define fixed data. Each \texttt{Variable} declaration selects one of the permitted domains in \texttt{Domain}. Since there are finitely many variables, these declarations determine a finitely presented decision domain
\[
D_A=D_1\times\cdots\times D_n,
\]
where every $D_i$ is binary, integer, real, bounded, or finite categorical. Each \texttt{Piecewise} declaration has finitely many breakpoints and therefore defines a finitely presented piecewise-linear function.

\paragraph{Feasible set.}
Let $C_1,\ldots,C_q$ be the \texttt{Constraint} declarations in $A$. Each $C_j$ contains a \texttt{Formula}; by Lemma~\ref{lem:term-formula-representation}, this formula returns a truth value for every valid $x\in D_A$. Let $h_j(x)$ denote that value. The hard constraints therefore define the feasible set
\[
F_A=\{x\in D_A:h_j(x)=\mathrm{true}
\text{ for every }j=1,\ldots,q\}.
\]
The same lemma ensures that each $h_j$ is constructed only from the expressions and conditions permitted in $\mathcal O$.

\paragraph{Base objective.}
If the \texttt{Objective} contains an expression, Lemma~\ref{lem:term-formula-representation} implies that it defines a scalar function $f_0$ on $D_A$. To express every objective as a minimization ordering, define
\[
\bar f_0(x)=
\begin{cases}
f_0(x), & \text{if the sense is \texttt{minimize}},\\
-f_0(x), & \text{if the sense is \texttt{maximize}},\\
0, & \text{if the sense is \texttt{feasibility}}.
\end{cases}
\]
Thus, $\bar f_0$ induces exactly the ordering specified by the \texttt{Objective}.

\paragraph{Soft constraints.}
Let $S_1,\ldots,S_m$ be the \texttt{SoftConstraint} declarations. By Lemma~\ref{lem:term-formula-representation}, each condition and guard is Boolean-valued and each explicit violation is a well-sorted expression. Under the interpretation of \texttt{SoftConstraint} described above, these fields define a nonnegative violation function $v_i$ on $D_A$.

If the aggregation mode is \texttt{weighted\_sum}, well-formedness supplies a nonnegative weight $w_i$ for each $v_i$. The program orders feasible solutions by the scalar function
\[
\bar f_0(x)+\sum_{i=1}^{m}w_i v_i(x),
\]
which is a weighted-penalty objective of the form permitted in $\mathcal O$.

If the aggregation mode is \texttt{lexicographic}, the priority fields and \texttt{PreferenceRelation} declarations place the finitely many soft constraints into finitely many consistently ordered levels. Number these levels from most to least important and define
\[
V_k(x)=\sum_{i:\mathrm{Priority}(i)=k}v_i(x).
\]
The program then orders feasible solutions lexicographically by
\[
\bigl(V_1(x),\ldots,V_K(x),\bar f_0(x)\bigr),
\]
which is a finite lexicographic objective of the form permitted in $\mathcal O$. If $A$ has no soft constraints, its ordering is determined by $\bar f_0$ alone.

Therefore, $A$ defines a finitely presented problem with an admissible decision domain, a feasible set defined by permitted conditions, and a base, weighted-penalty, or finite lexicographic objective. Hence,
\[
\llbracket A\rrbracket\in\mathcal O.
\]
\end{proof}

\begin{theorem}[Expressive completeness]
\label{thm:grammar-expressive-completeness}
For every $O\in\mathcal O$, there exists $A\in\mathcal A$ such that
\[
\llbracket A\rrbracket=O.
\]
\end{theorem}

\begin{proof}
Fix $O\in\mathcal O$. We construct an AST $A\in\mathcal A$ and show that it represents the same optimization problem as $O$.

\paragraph{Decision domain.}
Because $O$ is finitely presented, it contains finitely many index domains, data quantities, decision quantities, and piecewise-linear functions. Represent each finite index domain by a \texttt{Set} declaration and each data quantity by a \texttt{Parameter} declaration with the same value and domain. Represent each decision quantity by a \texttt{Variable} declaration with the same domain, indices, and bounds. Represent each piecewise-linear function by a \texttt{Piecewise} declaration with the same input, breakpoints, and values or marginal rates. These declarations reproduce the decision domain $D_O$ of $O$.

\paragraph{Feasible set.}
Let $h_1,\ldots,h_q$ be the hard constraints of $O$. By Lemma~\ref{lem:term-formula-representation}, each $h_j$ has a \texttt{Formula} representation $\widehat h_j$ that returns the same truth value as $h_j$ for every valid assignment. Place each $\widehat h_j$ in a \texttt{Constraint} declaration. The resulting constraints define
\[
\{x\in D_O:\widehat h_j(x)=\mathrm{true}
\text{ for every }j\}
=
\{x\in D_O:h_j(x)=\mathrm{true}
\text{ for every }j\},
\]
so they reproduce the feasible set of $O$.

\paragraph{Base objective.}
If $O$ has a base objective $f_0$, Lemma~\ref{lem:term-formula-representation} supplies an \texttt{Expr} that returns $f_0(x)$ for every valid assignment. Place this expression in an \texttt{Objective} with the same minimization or maximization sense as $O$. If $O$ is a feasibility problem, use \texttt{feasibility} with no objective expression. The resulting \texttt{Objective} therefore reproduces the base ordering of $O$.

\paragraph{Soft constraints.}
For each soft constraint $i$, use Lemma~\ref{lem:term-formula-representation} to represent its satisfaction condition by a \texttt{Formula} and its nonnegative violation function $v_i$ by an \texttt{Expr}. Place these representations in the \texttt{condition} and \texttt{violation} fields of a \texttt{SoftConstraint} declaration.

If $O$ uses weighted penalties, set \texttt{aggregation} to \texttt{weighted\_sum} and place the corresponding nonnegative coefficient $w_i$ in \texttt{weight}. The resulting declarations reproduce the objective
\[
\bar f_0(x)+\sum_i w_i v_i(x).
\]

If $O$ uses lexicographic priorities, set \texttt{aggregation} to \texttt{lexicographic} and place the level of each soft constraint in \texttt{priority}. When preference relations fully determine the ordered tiers, encode
the ordering with \texttt{PreferenceRelation} declarations. Otherwise,
specify the tiers explicitly in \texttt{priority}. These declarations reproduce the ordered objective vector
\[
\bigl(V_1(x),\ldots,V_K(x),\bar f_0(x)\bigr),
\qquad
V_k(x)=\sum_{i:\mathrm{Priority}(i)=k}v_i(x).
\]

\paragraph{Assembly.}
Collect the constructed declarations and objective in a \texttt{Program}. All components use productions of $G$, and their types, domains, and aggregation fields agree with their declarations; therefore, the resulting AST $A$ belongs to $\mathcal A$. Its declarations reproduce the decision domain of $O$, its hard constraints reproduce the feasible set of $O$, and its objective and soft constraints reproduce the ordering of feasible solutions in $O$. Under the stated notion of equivalence,
\[
\llbracket A\rrbracket=O.
\]
\end{proof}

\begin{corollary}[Representational adequacy]
\label{cor:grammar-representational-adequacy}
The semantic image of $G$ is exactly $\mathcal O$:
\[
\{\llbracket A\rrbracket:A\in\mathcal A\}=\mathcal O.
\]
\end{corollary}

\begin{proof}
Theorem~\ref{thm:grammar-well-definedness} gives the inclusion from left to right, and Theorem~\ref{thm:grammar-expressive-completeness} gives the reverse inclusion.
\end{proof}

\section{Prompts and Inference Protocol}
\label{app:prompts}

The following are the exact prompt templates used in the reported comparison.
Text in double braces is replaced deterministically before the request is sent.

\subsection{ASTOR}

\begin{promptblock}{AST prompt template}
You are an expert in mathematical optimization. Your task is to formalize a
natural-language optimization problem as an optimization model.

You will be given a formal grammar that generates well-formed optimization
models. Using this grammar, construct an abstract syntax tree representing the
model the problem describes.

The problem may contain hard constraints, which
must always be satisfied, and soft constraints, which may be violated. It is
your job to determine the hardness or strength of each constraint, which may
be indicated by modal words of varied strength or by explicit comparison. For
problems with soft constraints, decide whether the optimization is
lexicographic or uses a weighted-sum penalty objective.

Here is the formal grammar:

{{grammar}}

The abstract syntax tree is serialized as JSON. You must strictly obey this
schema: it is parsed deterministically into solver code, and a response that
deviates from it is discarded unread.

{{ast_json_schema}}

TASK SEMANTICS:
{{benchmark_semantics_contract}}

BENCHMARK COMPATIBILITY CONTRACT:
{{benchmark_compatibility_contract}}

STEP 1 — Declare every symbol. Classify each as exactly one of:
A = set: a collection the problem indexes over ("the plants", "each product")
B = parameter: fixed data the optimizer cannot change (costs, capacities, demands)
C = decision variable: a quantity the optimizer chooses (how much to ship, whether to open)

STEP 2 — Express every requirement, preference, and objective in the appropriate prescribed form within
the serialized AST JSON, using ONLY the symbols declared in Step 1.

Do not invent or change problem information.
End your response with the complete serialized AST in a single ```json fenced
block: exactly one, containing the entire JSON object. Only that block is
parsed.
\end{promptblock}

The request appends the instance using the exact boundary

\begin{promptblock}{AST request assembly}
{{rendered_ast_prompt}}

PROBLEM:
{{problem_text_and_authoritative_data}}
\end{promptblock}

\begin{longtable}{|p{0.28\linewidth}|p{0.64\linewidth}|}
\caption{Deterministic substitutions in the AST prompt.}\label{tab:prompt-substitutions}\\
\hline
\textbf{Placeholder} & \textbf{Source} \\
\hline
\endfirsthead
\hline
\textbf{Placeholder} & \textbf{Source} \\
\hline
\endhead
\nolinkurl{grammar} & The grammar reproduced in Appendix~\ref{app:formal-adequacy}, inserted verbatim from the artifact at runtime. \\
\hline
\nolinkurl{ast_json_schema} & The strict JSON response schema, inserted in full at runtime and retained in the artifact rather than repeated in print. \\
\hline
\nolinkurl{benchmark_semantics_contract} & The benchmark-specific scoring semantics reproduced below. \\
\hline
\nolinkurl{benchmark_compatibility_contract} & The benchmark-specific output interface reproduced below. \\
\hline
\end{longtable}

\subsection{Direct}

\begin{promptblock}{Direct prompt template}
You are an expert in mathematical optimization. Your task is to formalize a
natural-language optimization problem as an optimization model.

Solve the optimization problem yourself. Use only the supplied problem
and authoritative instance data. Do not defer to a solver.

The problem may contain hard constraints, which must always be satisfied, and
soft constraints, which may be violated. It is your job to determine the
hardness or strength of each constraint, which may be indicated by modal
words of varied strength or by explicit comparison. For problems with soft
constraints, decide whether the optimization is lexicographic or uses a
weighted-sum penalty objective.

Your response must conform exactly to this JSON Schema:
{{direct_output_schema}}

When status is optimal, report a numeric objective_value and one concrete
optimal assignment in variable_values. When status is infeasible or unbounded,
use null objective_value and an empty variable_values object.

{{benchmark_semantics_contract}}

PROBLEM
{{problem_text}}

AUTHORITATIVE INSTANCE DATA
{{direct_instance_data}}
\end{promptblock}

The inserted schema is a closed JSON object with exactly three fields:
\texttt{status} in \{\texttt{optimal}, \texttt{infeasible},
\texttt{unbounded}\}, a numeric-or-null \texttt{objective\_value}, and a
\texttt{variable\_values} object whose values are numeric.

\subsection{Codegen}

\begin{promptblock}{Codegen prompt template}
You are an expert in mathematical optimization. Your task is to formalize a
natural-language optimization problem as an optimization model.

Write one complete Python program that formulates and solves the supplied
optimization problem with gurobipy. Do not solve in prose.

The problem may contain hard constraints, which must always be satisfied, and
soft constraints, which may be violated. It is your job to determine the
hardness or strength of each constraint, which may be indicated by modal
words of varied strength or by explicit comparison. For problems with soft
constraints, decide whether the optimization is lexicographic or uses a
weighted-sum penalty objective.

INSTANCE_DATA CONTRACT
{{instance_data_contract}}

TASK SEMANTICS
{{benchmark_semantics_contract}}

BENCHMARK COMPATIBILITY CONTRACT
{{benchmark_compatibility_contract}}

The primary Gurobi model object must be named `model`; call model.optimize().
End by copying this reporting protocol verbatim and unchanged.

_STATUS_NAMES = {
 GRB.OPTIMAL: "optimal",
 GRB.INFEASIBLE: "infeasible",
 GRB.UNBOUNDED: "unbounded",
 GRB.INF_OR_UNBD: "infeasible_or_unbounded",
}
if model.Status == GRB.INF_OR_UNBD:
 model.Params.DualReductions = 0
 model.reset()
 model.optimize()
print(f"STATUS = {_STATUS_NAMES.get(model.Status, model.Status)}")
if model.Status == GRB.OPTIMAL:
 print(f"OBJECTIVE = {model.ObjVal}")
 for _variable in model.getVars():
 print(f"VARIABLE {_variable.VarName} = {_variable.X}")

Do not invent or change problem information. End your response with the complete program in a single ```python fenced
block: exactly one. Only that block is parsed.

PROBLEM
{{problem_text}}
\end{promptblock}

\subsection{Chain-of-Thought}

\begin{promptblock}{CoT prompt template}
The problem description is as follows:
{{problem_text}}

Please provide its mathematical model and Gurobipy code to solve this problem. The code you return should be a callable function whose parameters have default values extracted from the problem, and whose return value is the optimal objective function value of the model (if it exists), otherwise return None. Your response should include:

```model
[Your Model]
```

```python
[Your Code]
```

Now, Let's think step by step.

{{instance_data_contract}}
{{benchmark_semantics_contract}}
{{benchmark_compatibility_contract}}
{{reporting_protocol}}
\end{promptblock}

The CoT route uses the same reporting adapter as Codegen. The text inserted at
\texttt{reporting\_protocol} begins with the following compiler adapter and is
followed verbatim by the \texttt{\_STATUS\_NAMES} block in the Codegen prompt:

\begin{promptblock}{CoT compiler adapter}
Compiler adapter: although the upstream prompt describes a callable function, this project executes the returned Python block as a complete standalone program. Define the primary Gurobi model object at top level, name it `model`, and call model.optimize().
End the program by copying this reporting protocol verbatim and unchanged.
\end{promptblock}

\subsection{ORThought and Chain-of-Experts}

Both multi-agent baselines run from their official implementations, ORThought at
commit \texttt{5293e248} of its public repository~\citep{orthought} and
Chain-of-Experts at commit \texttt{afb89bf6}~\citep{complexor}. Their prompt
prose is unchanged. Each agent prompt is followed by the same instance-data,
task-semantics, and compatibility contracts that Codegen receives, and the
prompt that emits the final program also carries the Codegen reporting protocol.
ORThought runs its model agent once and then up to three execution-conditioned
debug rounds, which see only technical execution errors. CoE runs the conductor
over the seven upstream experts for three selections, then the reducer, with the
optional reflection pass disabled. Both use the backbone, decoding settings, and
execution limits of Table~\ref{tab:frozen-settings}, so an ORThought attempt
makes one to four LLM requests and a CoE attempt makes seven. The verbatim
prompt files and adaptation notes are in the released code.

\subsection{Benchmark Contracts}

Direct receives the task semantics only. ASTOR, Codegen, CoT, ORThought, and CoE
also receive the route-specific compatibility text. These are the exact strings used for the
reported benchmarks.

\begin{promptblock}{Hard-constraint benchmarks}
TASK SEMANTICS:
Enforce every required choice with an equality or an activated lower bound; an upper bound alone is insufficient.

AST AND CODE COMPATIBILITY:
No benchmark-specific decision interface is required.
\end{promptblock}

\begin{promptblock}{COMPASS}
TASK SEMANTICS:
Select exactly one option. Report select[option_id] = 1 for the chosen option, using the supplied id verbatim; omitted entries count as 0.

AST COMPATIBILITY:
Declare binary decision variable select indexed by [Options] and select exactly one option. Bind every option-dependent soft-condition index within its formula.

CODE COMPATIBILITY:
Use Gurobi binary variables select[option] over Options and select exactly one option.
\end{promptblock}

\begin{promptblock}{TravelPlanner}
TASK SEMANTICS:
Select exactly one placement, one eligible transport option per leg, one eligible hotel per visited city, and one eligible restaurant per active meal slot; selected restaurants must be distinct. Evaluate preferences from most to least important; improvement on a less important preference cannot compensate for worsening a more important one, and consider total cost only afterward. Report the chosen plan as use_placement[placement_id], use_transport[transport_id], pick_hotel[placement_id,hotel_id], and assign_meal[meal_slot_id,restaurant_id] entries set to 1, using supplied ids verbatim; omitted entries count as 0.

AST COMPATIBILITY:
Candidate *_ok and *_active parameters define eligibility.

Use these binary decision variables exactly: use_placement[Placements], use_transport[TransportOptions], pick_hotel[Placements,Hotels], and assign_meal[MealSlots,Restaurants].

CODE COMPATIBILITY:
Candidate *_ok and *_active parameters define eligibility.

Use these Gurobi binary variable families exactly: use_placement[placement], use_transport[option], pick_hotel[placement,hotel], and assign_meal[slot,restaurant].
\end{promptblock}

\begin{promptblock}{LexiBench}
TASK SEMANTICS:
Enforce every required choice with an equality or an activated lower bound; an upper bound alone is insufficient. Evaluate preferences from most to least important; improvement on a less important preference cannot compensate for worsening a more important one, and apply any stated tie-break only afterward. Report exactly one indexed decision-variable family per supplied set, including every member with unselected members set to 0; report no auxiliary quantities.

AST COMPATIBILITY:
No benchmark-specific decision interface is required.

CODE COMPATIBILITY:
Create each decision-variable family with model.addVars(members, vtype=..., name="family") so Gurobi formats the brackets; never build a name with an f-string like f"family_{member}".
\end{promptblock}

\begin{promptblock}{Nurse Rostering}
TASK SEMANTICS:
Minimize the combined penalty for staffing-target deviations and unmet shift requests according to their stated severities. Report work[nurse,day,shift] = 1 for every worked shift, using supplied ids verbatim; omitted entries count as 0. Any numerical penalties you choose to use must be inferred from text and instantiated yourself. Never invent other problem data.

AST COMPATIBILITY:
Declare binary decision variable work indexed by [Nurses,Days,Shifts], where 1 means the nurse works that shift on that day. Auxiliary variables may use arbitrary names.

Declare every set and parameter listed in the external data contract exactly once, using those exact names and index domains, with values provided in EXTERNAL DATA CONTRACT declared null. You must instantiate any penalty weights with numbers or expressions.

CODE COMPATIBILITY:
Create work = model.addVars(nurses, days, shifts, vtype=GRB.BINARY, name="work"), where 1 means the nurse works that shift on that day. Auxiliary variables may use arbitrary names.
\end{promptblock}

For schema-bearing AST instances other than Nurse Rostering, the following sentence is appended verbatim
to the AST compatibility contract:

\begin{promptblock}{External-data binding}
Declare every set and parameter listed in the external data contract exactly once, using those exact names and index domains, with null external values.
\end{promptblock}

The instance-data contract in Codegen and CoT states that
\texttt{INSTANCE\_DATA} is a flat dictionary with no \texttt{sets} or
\texttt{parameters} wrapper. For schema-bearing instances, sets are lists,
indexed parameters use string or comma-joined keys, and scalar parameters are
bare numbers. For prose-only instances, the contract states verbatim:
``\texttt{INSTANCE\_DATA is an empty dict for this problem: every value you
need is stated in the problem text itself.}''

\subsection{Execution and Evaluation}

ASTOR makes one LLM request. The output includes exactly one JSON
block, repaired deterministically only when JSON syntax admits an unambiguous
local correction, checked against the response schema and symbol table, and
compiled deterministically to Gurobi. Direct is evaluated from its reported
assignment; Codegen and CoT execute their single returned Python program, and
ORThought and CoE execute the final program returned by their last agent. All
routes then use the same benchmark oracle. The frozen runs contain 790 scored
ASTOR attempts and no JSON-repair, error-retry, or rephrase API call artifacts;
thus each scored ASTOR attempt used one LLM request.

For numeric objectives, a prediction is correct when the solver reports the
gold status and the absolute error is at most
$\max\{10^{-4},10^{-5}\max(1,|z^\star|)\}$, where $z^\star$ is the gold-standard optimal objective value for the instance. Infeasible and unbounded gold
outcomes require an exact status match. Soft benchmarks additionally decode
the reported assignment, verify feasibility, and recompute the weighted
penalty or lexicographic tier vector; a lexicographic result must match every
tier and its final numeric tie-break.

\begin{table}[H]
\centering
\small
\caption{Frozen settings for the main experiment.}
\label{tab:frozen-settings}
\begin{tabular}{@{}p{0.34\linewidth}p{0.58\linewidth}@{}}
\hline
\textbf{Setting} & \textbf{Value} \\
\hline
Provider and model & Gemini 3.7 Flash; Nurse Rostering: Gemini 3.8 Flash \\
Decoding & Temperature $0$; maximum output 32,768 tokens (65,536 for Nurse Rostering) \\
Thinking & Provider default, except Nurse Rostering, where every route used a 32,768-token thinking budget \\
Sampling & One repeat per instance; seed $0$; blind prompt mode \\
External tools & Disabled \\
Structured output & Disabled; deterministic validation applied after generation \\
Repair & Deterministic JSON handling enabled; no additional repair request observed \\
Execution & Gurobi 13+; 300-second generated-program and solver limits (3,600 seconds for Nurse Rostering) \\
Context cache & Enabled for reusable ASTOR context; provider-reported cache use priced as cached input \\
Pricing & USD per million tokens: input $0.75$, cached input $0.075$, output $3.75$ \\
\hline
\end{tabular}
\end{table}

Nurse Rostering is the one exception to the default thinking configuration. Under the provider default, all routes had substantial difficulty with Nurse Rostering, so the reported Nurse Rostering runs use a higher thinking budget and token limit for all six routes.

\subsection{Ablation Definitions}

The component study changes one element of the full ASTOR condition at a time, holding the main experimental setup constant.
\emph{Without benchmark contract} replaces both benchmark-contract fields by
\texttt{None.}; and \emph{without deterministic repair} disables local JSON
repair while leaving the model request unchanged.
\emph{Direct + AST} is a prompt-only control rather than an ASTOR variant: the
Direct prompt with the ASTOR grammar of Appendix~\ref{app:formal-adequacy}
prepended and an instruction to construct the abstract syntax tree of the model
as an internal reasoning step before answering. Its output schema, benchmark
contract, instance data, decoding settings, and evaluation are those of Direct,
and nothing is compiled or solved. 

\subsection{Frozen Ablation Sample}
\label{app:ablation-sample}

Following the cost-controlled analyses in NEMO, OptMATH, and SeaKR, which use
restricted benchmark sets or sampled analysis subsets, we run ablations on a
frozen subset of the main evaluation cohort \citep{nemo,OptMATH,seakr}. We use
Hamilton largest-remainder allocation to assign exactly 100 of the 790
instances proportionally across the eight reported benchmarks, yielding 55
hard-constraint and 45 soft-constraint instances \citep{hamilton_method}. With master seed
\texttt{20260825}, we sample uniformly without replacement within each stratum;
the COMPASS allocation is further divided proportionally into ten hotel and one
complex instance. The same IDs are used for every ablation condition, and no
method output is used in their selection.

\begin{table}[H]
\centering
\caption{Proportional allocation of the frozen ablation sample.}
\label{tab:ablation-sample-allocation}
\small
\begin{tabular}{lrr@{\hspace{1.4em}}lrr}
\hline
Benchmark & Cohort & Sample & Benchmark & Cohort & Sample \\
\hline
ComplexOR & 18 & 2 & COMPASS & 85 & 11 \\
IndustryOR & 83 & 10 & TravelPlanner & 180 & 23 \\
LogiOR & 92 & 12 & LexiBench & 75 & 9 \\
NLP4LP & 242 & 31 & Nurse Rostering & 15 & 2 \\
\hline
Hard total & 435 & 55 & Soft total & 355 & 45 \\
\hline
\end{tabular}
\end{table}

\noindent\textbf{Frozen instance IDs.}
The identifiers of the 100 sampled instances, together with the per-stratum seeds, are provided as a machine-readable manifest in the anonymous repository.

\subsection{Artifact Index}

The anonymous repository contains the complete grammar and its JSON schema, the exact prompt templates for every method and for the Direct + AST control, the frozen 100-instance ablation manifest, the contract-construction code, the parser, validator, and compiler, the benchmark adapters and evaluators together with the SoftOR data, the frozen experiment configurations, the per-problem result records of the 790-instance evaluation and of the backbone ablation, the reporting inputs from which every table and figure is regenerated, and the test suite. 

\section{Error Annotations and Discussion}
\label{app:error-taxonomy}

For the error analysis in the main text, errors are classified as follows: a recorded parser,
compiler, solver-execution, or other harness failure is \emph{technical}. This
includes a non-successful formulation-conversion status. An attempt that passes
the automatic evaluator is \emph{correct}. Every remaining attempt is a
\emph{solution error}. Thus the error rates use all evaluated attempts as their
denominator, and the unplotted remainder in
Figures~\ref{fig:error-analysis-soft} and~\ref{fig:error-analysis} is correct.

For example, ComplexOR
\texttt{prob\_007} under Codegen returned no Python fence and was labeled
\emph{technical}. ComplexOR \texttt{prob\_008} under CoT executed normally but
returned objective $10$ against gold $7$ and was labeled a \emph{solution error}.

\subsection{Additional Error Analysis}

\subsubsection{Soft Constraints}

 On LexiBench, ASTOR returns hard-feasible solutions on all 75 instances, but six have incorrect preference outcomes. Five of these use weighted-sum aggregation with a cost objective, while the sixth assigns higher priority to preferences that the request explicitly says to sacrifice first. Among the local constraints in the 180 evaluated TravelPlanner instances, ASTOR treats 17 of 124 preferences as hard constraints and 5 of 116 hard constraints as preferences. Treating a preference as hard was harmless in 11 of the 17 cases, and treating a hard constraint as a preference never caused a failure.

ASTOR has 41 unsuccessful attempts on SoftOR: 27 on TravelPlanner, 6 on LexiBench, 7 on Nurse Rostering, and 1 on COMPASS. Of these, 1 is a technical failure and 40 are classified as solution errors by the evaluator. All seven unsuccessful Nurse Rostering attempts return hard-feasible rosters, with published-weight penalties between 1 and 4 points above the gold optimum. Lexicographic mistakes include wrong preference tiers or misdiagnosing a soft constraint as hard. Some examples: In the laptop instance, five equal wishes are followed by ``the cheaper one wins''; ASTOR kept all five wishes but summed them with price, and satisfied three instead of four. In the service-vehicle instance, the text ranks wishes as ``we would sooner run without $X$ than without $Y$''; ASTOR reversed the order. In the makerspace instance, the top wish is ``the one thing we will not trade''; ASTOR made it a hard constraint.

When a method gets a soft-constraint problem wrong, it is usually not because it missed a preference. It often treats one preference as too strict, or mixes the preferences together the wrong way. The baselines often end up minimizing cost instead of respecting the priority order, so a cheap plan wins even when it skips something more important. ASTOR's mistakes go the other way: it sometimes turns a preference into a hard rule, or lets cost outweigh a preference it should have kept.

\subsubsection{Hard Constraints}

On the hard-constraint benchmarks, a wrong answer is almost always a wrong objective value rather than a failed solve. A value better than gold means the method was able to achieve a lower value for a minimization objective or a higher value for a maximization objective than the annotated gold-standard optimum, and a value worse than gold means the opposite. Dropped constraints are the more common error for every method, by about two to one: Codegen returns 47 values better than gold and 22 worse, CoT 51 and 23, ORThought 55 and 22, CoE 46 and 24, and ASTOR 43 and 17, the fewest over-constrained formulations. Direct, which solves by hand, returns 71 better and 31 worse. Infeasible or unbounded formulations are rare, at most 3 per method. 

When a method gets a hard-constraint problem wrong, it is almost always because it left out a constraint, not because it added one that was not there. That makes the solver report a better answer than is actually allowed. It is rare for a method to end up with an infeasible or over-constrained model instead. The grammar does not stop constraints from getting dropped, but when one does get dropped, it is more likely to show up as an obvious error instead of a wrong answer that looks right.

\section{Benchmark Construction and Dataset Adaptations}
\label{app:benchmark-construction}

This appendix records the rules used to turn the source datasets into natural-language optimization problems and to compute their gold answers. Table~\ref{tab:softor-composition-appendix} summarizes the four SoftOR components, and the sections below detail each construction.

\begin{table}[H]
\centering
\small
\caption{SoftOR composition. Gold annotations are provided or verified by enumeration/solver optimization.}
\label{tab:softor-composition-appendix}
\begin{tabular}{lrlcl}
\toprule
Component & $n$ & Objective & Priority Signal & Gold construction \\
\midrule
Nurse Rostering & 15 & Weighted penalty & Explicit & Published optimum \\
COMPASS & 85 & Lexicographic & Explicit & Enumeration \\
TravelPlanner & 180 & Lexicographic & Modal & Enumeration \\
LexiBench & 75 & Lexicographic & Mixed & Enumeration/Solver Code \\
\bottomrule
\end{tabular}
\end{table}

The specific linguistic cues used by the first three datasets are depicted in Figure \ref{fig:priority-annotation}. These cues are further elaborated on in subsequent sections. 

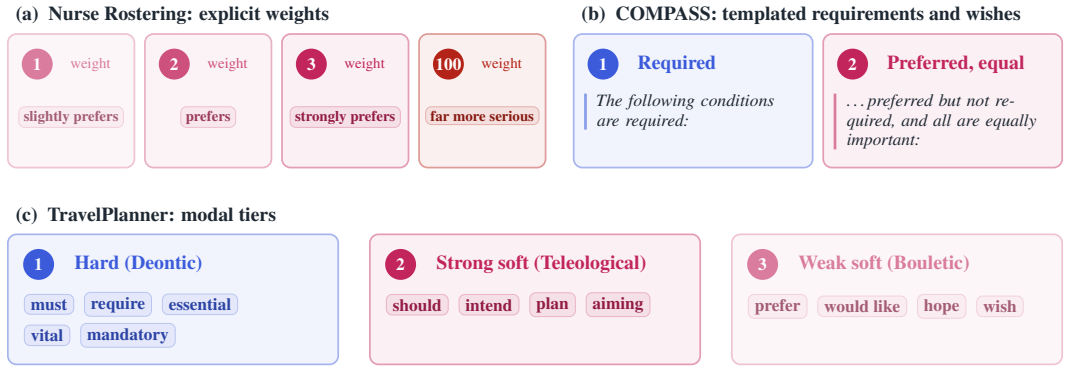
\begin{figure}[!htbp]
\centering
\resizebox{\textwidth}{!}{%
\begin{tikzpicture}[
 font=\sffamily,
 >=Latex,
 card/.style={draw=pipedark!22,line width=.5pt,rounded corners=2pt,fill=white,
 inner sep=2.4pt,align=left,font=\scriptsize},
 innerflow/.style={-{Latex[length=1.6mm,width=1.3mm]},line width=.7pt,draw=pipedark!55},
 pref/.style={circle,draw=acc2,fill=acc2!12,line width=.6pt,minimum size=3.8mm,inner sep=0pt,
 font=\tiny\bfseries,text=pipedark},
 ptitle/.style={font=\bfseries\footnotesize,text=pipedark,anchor=north west},
 panel/.style 2 args={fill=#1!7,draw=#1!45,line width=.7pt,rounded corners=3.5pt,inner sep=0pt,
 minimum width=#2,anchor=north west},
 badge/.style={circle,fill=#1,text=white,font=\bfseries\scriptsize,minimum size=4.4mm,inner sep=0pt},
 ctitle/.style={font=\bfseries\footnotesize,text=#1,anchor=west},
 chip/.style={rounded corners=2.4pt,fill=#1!14,draw=#1!35,line width=.4pt,text=#1!75!black,
 font=\scriptsize\bfseries,inner xsep=2.6pt,inner ysep=1.7pt},
 quote/.style={font=\scriptsize\itshape,text=pipedark,text width=27mm,align=left,inner sep=0pt},
 chipsmall/.style={chip=#1,font=\tiny\bfseries,inner xsep=2pt,inner ysep=1.5pt},
]

% ---------- (a) Nurse Rostering: explicit weights ----------
\node[ptitle] at (0,0) {(a)\ \ Nurse Rostering: explicit weights};
\newcommand{\weightcard}[5]{%
 % 1 name, 2 x, 3 colour, 4 weight, 5 phrase
 \node[panel={#3}{18mm},minimum height=19mm] (#1) at (#2,-5mm) {};
 \node[badge=#3] (#1num) at ([shift={(4.3mm,-4.4mm)}]#1.north west) {#4};
 \node[font=\tiny,text=#3,anchor=west] at ([xshift=1.2mm]#1num.east) {weight};
 \node[chipsmall=#3,anchor=north] at ([yshift=-10.4mm]#1.north) {#5};
}
\weightcard{w1}{0mm}{pipered!60!white}{1}{slightly prefers}
\weightcard{w2}{19.5mm}{pipered!80!white}{2}{prefers}
\weightcard{w3}{39mm}{pipered}{3}{strongly prefers}
\weightcard{w4}{58.5mm}{errorred}{100}{far more serious}

% ---------- (b) COMPASS: fixed templates ----------
\node[ptitle] at (80.5mm,0) {(b)\ \ COMPASS: templated requirements and wishes};
\node[panel={pipeblue}{34mm},minimum height=19mm] (c1) at (80.5mm,-5mm) {};
\node[badge=pipeblue] (c1num) at ([shift={(4.3mm,-4.4mm)}]c1.north west) {1};
\node[ctitle=pipeblue] at ([xshift=1.4mm]c1num.east) {Required};
\node[quote=pipeblue,anchor=north west] (c1q) at ([shift={(3.4mm,-8.6mm)}]c1.north west)
 {The following conditions are required:};
\draw[pipeblue!60,line width=1.1pt] ([xshift=-1.6mm]c1q.north west) -- ([xshift=-1.6mm]c1q.south west);
\node[panel={pipered}{34mm},minimum height=19mm] (c2) at (116mm,-5mm) {};
\node[badge=pipered] (c2num) at ([shift={(4.3mm,-4.4mm)}]c2.north west) {2};
\node[ctitle=pipered] at ([xshift=1.4mm]c2num.east) {Preferred, equal};
\node[quote=pipered,anchor=north west] (c2q) at ([shift={(3.4mm,-8.6mm)}]c2.north west)
 {\ldots preferred but not required, and all are equally important:};
\draw[pipered!60,line width=1.1pt] ([xshift=-1.6mm]c2q.north west) -- ([xshift=-1.6mm]c2q.south west);

% ---------- (c) TravelPlanner tier cue mapping ----------
\node[ptitle] at (0,-28.5mm) {(c)\ \ TravelPlanner: modal tiers};
\newcommand{\tierhead}[5]{%
 % 1 name, 2 x, 3 colour, 4 number, 5 title
 \node[panel={#3}{47mm},minimum height=18.5mm] (#1) at (#2,-33.5mm) {};
 \node[badge=#3] (#1num) at ([shift={(4.5mm,-4.4mm)}]#1.north west) {#4};
 \node[ctitle=#3] at ([xshift=1.6mm]#1num.east) {#5};
}
\tierhead{t1}{0mm}{pipeblue}{1}{Hard (Deontic)}
\node[chip=pipeblue,anchor=north west] (t1c1) at ([shift={(2.4mm,-8.6mm)}]t1.north west) {must};
\node[chip=pipeblue,anchor=west] (t1c2) at ([xshift=1.4mm]t1c1.east) {require};
\node[chip=pipeblue,anchor=west] (t1c3) at ([xshift=1.4mm]t1c2.east) {essential};
\node[chip=pipeblue,anchor=north west] (t1c4) at ([yshift=-1.4mm]t1c1.south west) {vital};
\node[chip=pipeblue,anchor=west] (t1c5) at ([xshift=1.4mm]t1c4.east) {mandatory};
\tierhead{t2}{51.5mm}{pipered}{2}{Strong soft (Teleological)}
\node[chip=pipered,anchor=north west] (t2c1) at ([shift={(2.4mm,-8.6mm)}]t2.north west) {should};
\node[chip=pipered,anchor=west] (t2c2) at ([xshift=1.4mm]t2c1.east) {intend};
\node[chip=pipered,anchor=west] (t2c3) at ([xshift=1.4mm]t2c2.east) {plan};
\node[chip=pipered,anchor=west] (t2c4) at ([xshift=1.4mm]t2c3.east) {aiming};
\tierhead{t3}{103mm}{pipered!60!white}{3}{Weak soft (Bouletic)}
\node[chip=pipered!60!white,anchor=north west] (t3c1) at ([shift={(2.4mm,-8.6mm)}]t3.north west) {prefer};
\node[chip=pipered!60!white,anchor=west] (t3c2) at ([xshift=1.4mm]t3c1.east) {would like};
\node[chip=pipered!60!white,anchor=west] (t3c3) at ([xshift=1.4mm]t3c2.east) {hope};
\node[chip=pipered!60!white,anchor=west] (t3c4) at ([xshift=1.4mm]t3c3.east) {wish};

\end{tikzpicture}%
}
\caption{Priority annotation in SoftOR instances constructed from source data.}
\label{fig:priority-annotation}
\end{figure}

\subsection{TravelPlanner Priority Rules}
\label{app:modal-protocol}

TravelPlanner uses modal language to distinguish requirements from preferences. We map deontic, teleological, and bouletic expressions to three optimization levels \citep{teleogical,corrob1,corrob2,modal_bench}.

\begin{longtable}{|p{0.08\linewidth}|p{0.20\linewidth}|p{0.62\linewidth}|}
\caption{TravelPlanner modal priority rules.}\label{tab:travelplanner-modal-rules}\\
\hline
\textbf{Tier} & \textbf{Interpretation} & \textbf{Common cues} \\
\hline
\endfirsthead
\hline
\textbf{Tier} & \textbf{Interpretation} & \textbf{Common cues} \\
\hline
\endhead
1 & Hard constraint & \emph{must, require, need, essential, mandatory, cannot, non-negotiable, ensure, will not, do not plan, avoid, not via} \\
\hline
2 & Strong preference & \emph{should, important, intend to, plan to, aiming for} \\
\hline
3 & Weak preference & \emph{prefer, would like, appreciate, hope, wish, interested in, ideally, open to, prefer not, rather not} \\
\hline
\end{longtable}

Budget, party size, origin, destination, and the original TravelPlanner validity rules are always hard constraints. Other local constraints are classified using the modal cue in the same clause. Tier 1 takes precedence when several cues match, except that \emph{prefer not} and \emph{rather not} remain weak preferences. Unrecognized expressions default to Tier 3 and are flagged for review. This procedure recognizes a cue for all 180 instances used in our experiments. 

Let $x$ range over candidate plans, let $v_i(x)\in\{0,1\}$ indicate that plan $x$ violates local constraint $i$, let $\mathrm{Priority}(i)\in\{1,2\}$ be its strong or weak tier (tiers 2 and 3 of Table~\ref{tab:travelplanner-modal-rules}; tier 1 becomes a hard constraint $h_j$), and let $f_0(x)$ be the plan's cost. The gold objective is
\begin{align*}
\operatorname*{lex\,min}_{x}\quad & \bigl(V_1(x),\,V_2(x),\,f_0(x)\bigr)\\
\text{s.t.}\quad & h_j(x)\le 0\quad\forall j,\qquad V_k(x)=\sum_{i:\mathrm{Priority}(i)=k} v_i(x).
\end{align*}
A prediction is correct if it attains the same objective vector as the gold plan. It need not select the same plan when several plans are optimal.

The original TravelPlanner authors created 225 total instances, which together have 300 hard and soft constraints. We only use 180 of these instances (the subset they labeled "validation"), which together have 240 hard and soft constraints.

\subsection{LexiBench Construction}

LexiBench contains five instances for every combination of three decision types and five preference structures, giving $3\times5\times5=75$ problems.

\begin{longtable}{|p{0.20\linewidth}|p{0.70\linewidth}|}
\caption{LexiBench decision types.}\label{tab:lexibench-decision-types}\\
\hline
\textbf{Decision type} & \textbf{Decision space} \\
\hline
\endfirsthead
\hline
\textbf{Decision type} & \textbf{Decision space} \\
\hline
\endhead
Atomic & Select one of 12 supplied options. \\
\hline
Assembly & Select one element from each of three five-element sets, giving 125 possible combinations. \\
\hline
Constructive & Optimize over integer decision variables and an implicitly defined feasible region. \\
\hline
\end{longtable}

\begin{longtable}{|p{0.29\linewidth}|p{0.47\linewidth}|p{0.14\linewidth}|}
\caption{LexiBench preference structures and difficulty labels.}\label{tab:lexibench-preference-structures}\\
\hline
\textbf{Preference structure} & \textbf{Property tested} & \textbf{Difficulty} \\
\hline
\endfirsthead
\hline
\textbf{Preference structure} & \textbf{Property tested} & \textbf{Difficulty} \\
\hline
\endhead
Direct priority dominance & One preference level dominates another. & Easy \\
\hline
Equal-level feature counting & All preferences have equal priority. & Easy \\
\hline
One-high-many-low & Several lower preferences cannot replace one higher preference. & Medium \\
\hline
Deep ladder & Preferences form four or five ordered levels. & Hard \\
\hline
Sparse transitive ordering & The full order must be inferred from pairwise comparisons. & Hard \\
\hline
\end{longtable}

LexiBench uses the same modal words in Table \ref{tab:travelplanner-modal-rules} to signal preferences, with the same tier structure. It also uses explicit words "more important," "less important," and "above all" to get more than three tiers.

Difficulty depends only on preference structure, not decision type. The benchmark therefore contains 30 easy, 15 medium, and 30 hard problems. These labels and the gold annotations are not shown to the model.

Let $v_i(x)\in\{0,1\}$ indicate that decision $x$ violates preference $i$, let $\mathrm{Priority}(i)$ be its annotated level, ordered from highest to lowest, and let $f_0(x)$ be the stated tie-break objective, omitted when none is stated. Hard constraints $h_j$ come from budget, capacity, or required attributes marked by deontic language in the natural language instance. The gold objective is
\begin{align*}
\operatorname*{lex\,min}_{x}\quad & \bigl(V_1(x),\ldots,V_K(x),f_0(x)\bigr)\\
\text{s.t.}\quad & h_j(x)\le 0\quad\forall j,\qquad V_k(x)=\sum_{i:\mathrm{Priority}(i)=k} v_i(x).
\end{align*}

We certify each gold vector with Gurobi. We optimize the highest preference level, fix its optimal value, and repeat for each lower level. Any ordinary objective or tie-break is optimized last. Every stage must be optimal with zero MIP gap. A frozen manifest records the build seed, benchmark hash, solver version, and successful verification of all 75 problems.

\subsection{Nurse Rostering}
\label{app:nrp-construction}

We convert the published request weights into qualitative language.

\begin{longtable}{|p{0.25\linewidth}|p{0.65\linewidth}|}
\caption{Nurse Rostering preference phrases.}\label{tab:nrp-preference-phrases}\\
\hline
\textbf{Published weight} & \textbf{Natural-language phrase} \\
\hline
\endfirsthead
\hline
\textbf{Published weight} & \textbf{Natural-language phrase} \\
\hline
\endhead
1 & \emph{slightly prefers} \\
\hline
2 & \emph{prefers} \\
\hline
3 & \emph{strongly prefers} \\
\hline
100 & \emph{far more serious} \\
\hline
\end{longtable}

The included instances assign penalties of $100$ to understaffing and $1$ to overstaffing. We express this by stating that understaffing is ``far more serious'' than overstaffing. The numerical penalties are not shown to the model and are only used to evaluate the final roster produced by each method. If the route chooses to use weights in creating a roster, it must create its own weights. Once the roster is produced, we only use the original numerical penalties for evaluation, but they are never used in generating a solution. 

Note that the request weights 1, 2, 3, 100 are the numbers chosen by the original authors of Nurse Rostering at SchedulingBenchmarks, but those numbers are not unique in determining the preference order. For example, 3, 14, 86, and 1592 would create the same ordering. As a result, for our newly constructed natural language instances which \textit{only} signal request strength through linguistic cues, it is precisely the task of each method to be able to infer and instantiate numerical weights which reflect the strength of the given language. In the weighted penalty experiments we never show the original weights of 1, 2, 3, 100 to the models, and allow them to choose their own numbers or whether to use weights at all. After doing this, they will produce a nurse roster for this problem, which specifies which nurses work on which days using indicator decision variables. We only use the original weights 1, 2, 3, 100 after the decision is made for the sake of evaluation: to accrue the penalty of each roster by noting when the chosen decision violates staffing and shift preferences, and multiplying by their weights in accordance with the original SchedulingBenchmarks gold-standard calculation. 

Let $x_{nds}\in\{0,1\}$ indicate that nurse $n$ works shift $s$ on day $d$, and let $r_{ds}$ be the number of nurses required on day $d$ and shift $s$, as specified by the instance's cover requirement for each $(d,s)\in\{\text{days}\}\times\{\text{shifts}\}$. Each shift request $i$ names a nurse, day, and shift $(n_i,d_i,s_i)$ together with the phrase that carries its weight. The gold roster solves the weighted-penalty problem
\begin{align*}
\min_{x}\quad & V_1(x)+2\,V_2(x)+3\,V_3(x)+100\,U(x)+O(x)\\
\text{s.t.}\quad & h_j(x)\le 0\quad\forall j,\\[2pt]
\text{Understaffing:}\quad & U(x)=\textstyle\sum_{(d,s)}\bigl(\text{required}_{ds}-\text{assigned}_{ds}(x)\bigr)^{+},\\
\text{Overstaffing:}\quad & O(x)=\textstyle\sum_{(d,s)}\bigl(\text{assigned}_{ds}(x)-\text{required}_{ds}\bigr)^{+},
\end{align*}
where $(\cdot)^{+}=\max\{0,\cdot\}$, the sums run over days $d$ and shifts $s$, $\text{required}_{ds}=r_{ds}$, $\text{assigned}_{ds}(x)=\sum_n x_{nds}$, and the hard constraints $h_j$ encode the scheduling rules of the source benchmark. $V_k(x)$ counts the violated requests phrased \emph{slightly prefers} ($k=1$), \emph{prefers} ($k=2$), and \emph{strongly prefers} ($k=3$); a request to work a shift is violated when the nurse is not assigned to it, and a request not to work a shift is violated when the nurse is. This is the weighted-penalty form of Section~\ref{eq:soft-objectives1} with no base objective: request weights $1,2,3$, understaffing weight $100$, and overstaffing weight $1$.
Each predicted roster is checked against the original hard scheduling rules described below and rescored using this objective. The published optimal objective is used as the gold value.

Table~\ref{tab:nrp-partial-credit} reports partial-credit measures. \emph{Hard Constraints Satisfied} is calculated separately for each roster. Each nurse is evaluated on nine hard constraint categories, which form the $h_j$ in the formulation above: at most one shift per day, forbidden shift successions, limits on assignments to each shift type, total working-time bounds, fixed days off, maximum consecutive working days, minimum consecutive working days, minimum consecutive days off, and maximum weekends worked. A nurse receives one point for a category if and only if the rule is satisfied throughout the planning horizon. The roster score is the percentage of all nurse--category pairs receiving one point. Following the benchmark, blocks extending beyond the start or end of the horizon are excluded from the two minimum-consecutive-length checks.

The remaining measures are also calculated separately for each roster:
\[
\text{Understaffing}
=
\frac{\text{missing required nurse--shift positions}}
 {\text{total required nurse--shift positions}}.
\]
\[
\text{Overstaffing}
=
\frac{\text{assignments above staffing requirements}}
 {\text{total required nurse--shift positions}}.
\]
\[
\text{Weighted Preference Violation}
=
\frac{\text{total weight of violated shift requests}}
 {\text{total weight of all shift requests}}.
\]
Understaffing and overstaffing may both be nonzero because different shifts can fall below and exceed their requirements. Table~\ref{tab:nrp-partial-credit} reports the complements of these three measures, so that higher is better in every column: staffing coverage is one minus understaffing or overstaffing, and weighted preference satisfied is one minus weighted preference violation. Each of the 15 instances receives equal weight in its method's reported average, and an attempt with no decodable roster counts as satisfying no constraint.

\begin{table}[H]
\centering
\small
\caption{Nurse Rostering partial-credit measures (\%), averaged over all 15 instances. Higher is better.}
\vspace{0.2cm}
\label{tab:nrp-partial-credit}
\resizebox{0.88\linewidth}{!}{%
\begin{tabular}{lrrrr}
\toprule
Method & \shortstack{Hard Constraints\\Satisfied} & \shortstack{Staffing Coverage\\(under)} & \shortstack{Staffing Coverage\\(over)} & \shortstack{Weighted Preference\\Satisfied} \\
\midrule
Direct & 31.8 & 27.9 & 31.1 & 28.1 \\
Codegen & \textbf{100.0} & \textbf{94.8} & \textbf{97.9} & \textbf{88.0} \\
CoT & \textbf{100.0} & \textbf{94.8} & \textbf{97.9} & 87.9 \\
ORThought & 86.7 & 82.3 & 85.2 & 76.2 \\
CoE & \textbf{100.0} & \textbf{94.8} & \textbf{97.9} & 87.9 \\
ASTOR (Ours) & \textbf{100.0} & \textbf{94.8} & \textbf{97.9} & 87.9 \\
\bottomrule
\end{tabular}}
\end{table}

Direct returned five evaluable rosters: two satisfied all hard constraints and three violated at least one; the remaining ten attempts yielded no evaluable roster. Table~\ref{tab:nrp-direct-hard-errors} reports how many of the five evaluable rosters contained at least one violation from each observed rule family.

\begin{table}[H]
\centering
\small
\caption{Hard-constraint violations in Direct's five evaluable Nurse Rostering outputs.}
\label{tab:nrp-direct-hard-errors}
\begin{tabular}{lc}
\toprule
Hard-constraint family & Rosters with a violation \\
\midrule
Working-time bounds & 2/5 \\
Consecutive working-day bounds & 2/5 \\
Forbidden shift successions & 1/5 \\
Weekend limits & 2/5 \\
Consecutive days-off bounds & 2/5 \\
Shift-count limits & 2/5 \\
\bottomrule
\end{tabular}
\end{table}

Working-time, consecutive-working-day, consecutive-days-off, weekend, and shift-count limits were each violated in two of the five evaluable rosters.

Currently, the version of Nurse Rostering we use has each route infer and instantiate its own penalty weights, and only uses author-published weights for calculating penalties for the final scoring. We chose this setup because our main question is whether methods can turn natural-language preferences into decisions. Supplying the numerical weights before optimization would provide part of the preference interpretation that we want the methods to perform. As an ablation, we ran a version of Nurse Rostering that allowed methods to see weights before decision-making, using the same experimental settings as the main Nurse Rostering experiment. Direct sees numerical weights associated with natural language directly in the prompt, and all other routes use variable placeholders for each weight, which are deterministically replaced with the actual weights 1, 2, 3, or 100 when the code is generated, before running it to find an optimal solution. The results are presented in the Table \ref{old-nrp-instance-table}, which should be read in the same way as Table \ref{tab:nrp-penalty-above-gold}.

The task-semantics contract for this supplied-weight setting is reproduced verbatim below.
% Source: src/benchmark_prompting.py, benchmark_semantics_contract, commit d917591e.
\begin{promptblock}{Nurse Rostering with supplied weights: task-semantics contract}
Minimize the combined penalty for staffing-target deviations and unmet shift requests according to their stated severities. Report work[nurse,day,shift] = 1 for every worked shift, using supplied ids verbatim; omitted entries count as 0.
\end{promptblock}

\begin{table}[!htbp]
\centering
\definecolor{passgreen}{HTML}{22834A}
\definecolor{failred}{HTML}{B52D36}
\caption{Performance relative to gold-standard optimum on the 15 Nurse Rostering instances.}
\label{old-nrp-instance-table}

\fontsize{7.5}{10}\selectfont
\setlength{\tabcolsep}{2pt}
\renewcommand{\arraystretch}{1.1}
\begin{tabular*}{\linewidth}{@{\extracolsep{\fill}}lrrrrrrrrrrrrrrr@{}}
\toprule
Method & 1 & 2 & 3 & 4 & 5 & 6 & 7 & 8 & 9 & 10 & 11 & 12 & 13 & 14 & 15 \\
\midrule
Direct
& \textcolor{failred}{FAIL} & \textcolor{failred}{FAIL}
& \textcolor{failred}{FAIL} & \textcolor{failred}{FAIL}
& \textcolor{failred}{FAIL} & \textcolor{failred}{FAIL}
& \textcolor{failred}{FAIL} & \textcolor{failred}{FAIL}
& \textcolor{failred}{FAIL} & \textcolor{failred}{FAIL}
& \textcolor{failred}{FAIL} & \textcolor{failred}{FAIL}
& \textcolor{failred}{FAIL} & \textcolor{failred}{FAIL}
& \textcolor{failred}{FAIL} \\
Codegen
& \textcolor{passgreen}{0} & \textcolor{passgreen}{0}
& \textcolor{passgreen}{0} & \textcolor{passgreen}{0}
& \textcolor{passgreen}{0} & \textcolor{passgreen}{0}
& \textcolor{passgreen}{0} & \textcolor{failred}{FAIL}
& \textcolor{passgreen}{0} & \textcolor{passgreen}{0}
& \textcolor{passgreen}{0} & \textcolor{passgreen}{0}
& \textcolor{passgreen}{0} & \textcolor{passgreen}{0}
& \textcolor{passgreen}{0} \\
CoT
& \textcolor{passgreen}{0} & \textcolor{passgreen}{0}
& \textcolor{passgreen}{0} & \textcolor{passgreen}{0}
& \textcolor{passgreen}{0} & \textcolor{passgreen}{0}
& \textcolor{failred}{FAIL} & \textcolor{failred}{FAIL}
& \textcolor{passgreen}{0} & \textcolor{passgreen}{0}
& \textcolor{passgreen}{0} & \textcolor{passgreen}{0}
& \textcolor{passgreen}{0} & \textcolor{passgreen}{0}
& \textcolor{passgreen}{0} \\
ORThought
& \textcolor{passgreen}{0} & \textcolor{passgreen}{0}
& \textcolor{passgreen}{0} & \textcolor{passgreen}{0}
& \textcolor{passgreen}{0} & \textcolor{passgreen}{0}
& \textcolor{failred}{FAIL} & \textcolor{failred}{FAIL}
& \textcolor{passgreen}{0} & \textcolor{passgreen}{0}
& \textcolor{passgreen}{0} & \textcolor{passgreen}{0}
& \textcolor{passgreen}{0} & \textcolor{passgreen}{0}
& \textcolor{passgreen}{0} \\
CoE
& \textcolor{passgreen}{0} & \textcolor{passgreen}{0}
& \textcolor{passgreen}{0} & \textcolor{passgreen}{0}
& \textcolor{passgreen}{0} & \textcolor{passgreen}{0}
& \textcolor{failred}{FAIL} & \textcolor{failred}{FAIL}
& \textcolor{passgreen}{0} & \textcolor{passgreen}{0}
& \textcolor{passgreen}{0} & \textcolor{passgreen}{0}
& \textcolor{passgreen}{0} & \textcolor{passgreen}{0}
& \textcolor{passgreen}{0} \\
ASTOR
& \textcolor{failred}{FAIL} & \textcolor{passgreen}{0}
& \textcolor{passgreen}{0} & \textcolor{passgreen}{0}
& \textcolor{passgreen}{0} & \textcolor{passgreen}{0}
& \textcolor{failred}{FAIL} & \textcolor{passgreen}{0}
& \textcolor{passgreen}{0} & \textcolor{passgreen}{0}
& \textcolor{passgreen}{0} & \textcolor{passgreen}{0}
& \textcolor{passgreen}{0} & \textcolor{passgreen}{0}
& \textcolor{failred}{FAIL} \\
\bottomrule
\end{tabular*}
\end{table}

Supplying the numerical weights removes the need to infer preference strength, but still requires translating the scheduling rules into a valid executable formulation. Nurse rostering is still challenging in this setting: some methods violate hard constraints or fail to return an evaluable roster consistently.

\subsection{COMPASS}

We convert each structured COMPASS record into natural language using fixed templates. Hard constraints begin with:
\begin{quote}
The following conditions are required:
\end{quote}
Soft constraints begin with:
\begin{quote}
The following wishes are preferred but not required, and all are equally important:
\end{quote}
Each requirement and wish is rendered as a numbered sentence. The generated text is parsed back into structured form to verify that no constraint was changed or omitted.

Let $x$ range over the supplied candidates, let $v_i(x)\in\{0,1\}$ indicate that candidate $x$ fails wish $i$, and let the requirements be the hard constraints $h_j(x)\le 0$ (marked in the instance with ``require"). All wishes share one priority level, so the gold objective is
\begin{align*}
\operatorname*{lex\,min}_{x}\quad & \bigl(V_1(x)\bigr)\\
\text{s.t.}\quad & h_j(x)\le 0\quad\forall j,\qquad V_1(x)=\sum_i v_i(x).
\end{align*}
We recompute the optimum over all candidates and accept any feasible choice that attains it, including ties absent from the original annotation.

\subsection{ORThought Gold-Label Corrections}
\label{app:gold-corrections}

We found 16 ORThought records whose answer files reported an optimal status but contained no objective value~\citep{orthought}. We retained the original problem statements and reference formulations, reran the bundled solver code, and updated only the missing gold outcome. For solver results initially reported as infeasible or unbounded, we disabled dual reductions to distinguish the two statuses.

\begin{longtable}{|p{0.14\linewidth}|p{0.18\linewidth}|p{0.58\linewidth}|}
\caption{Audited corrections to ORThought gold outcomes.}\label{tab:orthought-gold-corrections}\\
\hline
\textbf{Dataset} & \textbf{Problem ID} & \textbf{Correction} \\
\hline
\endfirsthead
\hline
\textbf{Dataset} & \textbf{Problem ID} & \textbf{Correction} \\
\hline
\endhead
ComplexOR & \texttt{prob\_004} & Restored the objective value 250 from the accompanying sample and confirmed it with the reference formulation. \\
\hline
IndustryOR & \texttt{prob\_045} & Changed optimal with a null objective to infeasible; the reference scheduling formulation produces an irreducible infeasible subsystem. \\
\hline
NLP4LP & \texttt{prob\_002} & Changed optimal with a null objective to infeasible; the required advertising purchases exceed the available budget. \\
\hline
NLP4LP & \texttt{prob\_027} & Changed optimal with a null objective to infeasible; the vitamin D requirement is incompatible with the vitamin E upper bound. \\
\hline
NLP4LP & \texttt{prob\_050} & Changed optimal with a null objective to infeasible; the heating requirement cannot be met within the electricity limit. \\
\hline
NLP4LP & \texttt{prob\_056} & Changed optimal with a null objective to infeasible; the cooking-rate, oil-use, and grill-ordering requirements are incompatible. \\
\hline
NLP4LP & \texttt{prob\_057} & Changed optimal with a null objective to infeasible; the minimum bottle counts and ratio require more syrup than is available. \\
\hline
NLP4LP & \texttt{prob\_067} & Changed optimal with a null objective to infeasible; the crate minimum and ratio require more bananas than are available. \\
\hline
NLP4LP & \texttt{prob\_076} & Changed optimal with a null objective to infeasible; the slicing-rate, grease-use, and slicer-ordering requirements are incompatible. \\
\hline
NLP4LP & \texttt{prob\_082} & Changed optimal with a null objective to infeasible; the hydrogen requirement cannot be met under the pollution cap. \\
\hline
NLP4LP & \texttt{prob\_122} & Changed optimal with a null objective to infeasible; the minimum nighttime-pill requirement exceeds the available painkiller medicine. \\
\hline
NLP4LP & \texttt{prob\_125} & Changed optimal with a null objective to infeasible; the production requirements exceed the available labor hours. \\
\hline
NLP4LP & \texttt{prob\_134} & Changed optimal with a null objective to infeasible; the pill-production requirements exceed the available material. \\
\hline
NLP4LP & \texttt{prob\_178} & Changed optimal with a null objective to infeasible; the transport requirement cannot be met within the stated budget. \\
\hline
NLP4LP & \texttt{prob\_239} & Changed optimal with a null objective to unbounded; the model maximizes production subject only to lower bounds and has no limiting resource constraint. \\
\hline
NLP4LP & \texttt{prob\_243} & Changed optimal with a null objective to infeasible; the required bag ratio and minimum number of tiny bags exceed the energy budget. \\
\hline
\end{longtable}

\subsection{ComplexOR Key-Name Reconciliation}
\label{app:complexor-key-map}

Eight of ComplexOR's~\citep{complexor} eighteen problems describe their data using one set of symbols in the natural-language question and another in the accompanying instance JSON. We reconciled the two representations by renaming each instance-JSON key to the symbol the question text uses for it.

\begin{longtable}{|p{0.14\linewidth}|p{0.18\linewidth}|p{0.58\linewidth}|}
\caption{Renamed ComplexOR instance-JSON keys.}\label{tab:complexor-key-map}\\
\hline
\textbf{Dataset} & \textbf{Problem ID} & \textbf{Renamed key(s)} \\
\hline
\endfirsthead
\hline
\textbf{Dataset} & \textbf{Problem ID} & \textbf{Renamed key(s)} \\
\hline
\endhead
ComplexOR & \texttt{prob\_005} & \texttt{roll\_width}$\to$\texttt{RollWidth}, \texttt{widths}$\to$\texttt{Width}, \texttt{orders}$\to$\texttt{Orders}, \texttt{num\_patterns}$\to$\texttt{NumPatterns}, \texttt{num\_rolls\_width}$\to$\texttt{NumRollsWidth} \\
\hline
ComplexOR & \texttt{prob\_006} & \texttt{food\_set}$\to$\texttt{Foods}, \texttt{nutrient\_set}$\to$\texttt{Nutrients}, \texttt{food\_cost}$\to$\texttt{Cost}, \texttt{min\_food\_amount}$\to$\texttt{MinAmount}, \texttt{max\_food\_amount}$\to$\texttt{MaxAmount}, \texttt{min\_nutrient\_amount}$\to$\texttt{MinNutrient}, \texttt{max\_nutrient\_amount}$\to$\texttt{MaxNutrient}, \texttt{nutrient\_amount}$\to$\texttt{NutrientAmount} \\
\hline
ComplexOR & \texttt{prob\_007} & \texttt{cost}$\to$\texttt{Cost}, \texttt{n\_min}$\to$\texttt{MinRequirement}, \texttt{n\_max}$\to$\texttt{MaxRequirement}, \texttt{amt}$\to$\texttt{NutrientAmount} \\
\hline
ComplexOR & \texttt{prob\_008} & \texttt{jobs}$\to$\texttt{Jobs}, \texttt{machines}$\to$\texttt{Machines}, \texttt{proces\_time}$\to$\texttt{ProcesTime} \\
\hline
ComplexOR & \texttt{prob\_011} & \texttt{supply}$\to$\texttt{Supply}, \texttt{demand}$\to$\texttt{Demand}, \texttt{cost}$\to$\texttt{ShippingCost}, \texttt{limit}$\to$\texttt{Limit} \\
\hline
ComplexOR & \texttt{prob\_012} & \texttt{supply}$\to$\texttt{Supply}, \texttt{demand}$\to$\texttt{Demand}, \texttt{cost}$\to$\texttt{Cost}, \texttt{limit}$\to$\texttt{Limit} \\
\hline
ComplexOR & \texttt{prob\_014} & \texttt{supply}$\to$\texttt{Supply}, \texttt{demand}$\to$\texttt{Demand}, \texttt{rate}$\to$\texttt{Rate}, \texttt{limit}$\to$\texttt{Limit} \\
\hline
ComplexOR & \texttt{prob\_017} & \texttt{products}$\to$\texttt{Products}, \texttt{stages}$\to$\texttt{Stages}, \texttt{rate}$\to$\texttt{Rate}, \texttt{profit}$\to$\texttt{Profit}, \texttt{commit}$\to$\texttt{Commit}, \texttt{market}$\to$\texttt{Market}, \texttt{avail}$\to$\texttt{Available} \\
\hline
\end{longtable}

\end{document}